\documentclass[letterpaper]{article} 
\usepackage{aaai24}  
\usepackage{times}  
\usepackage{helvet}  
\usepackage{courier}  
\usepackage[hyphens]{url}  
\usepackage{graphicx} 
\usepackage{natbib}  
\usepackage{caption} 
\usepackage[ruled,vlined]{algorithm2e}

\usepackage{amsthm}
\newtheorem{assumption}{Assumption}
\newtheorem{proposition}{Proposition}
\newtheorem{theorem}{Theorem}
\newtheorem*{proposition*}{Proposition}
\newtheorem{lemma}[theorem]{Lemma}
\newtheorem*{lemma*}{Lemma}
\newtheorem{corollary}{Corollary}
\newtheorem*{corollary*}{Corollary}
\newtheorem*{remark}{Remark}
\usepackage{multirow}

\usepackage{caption}
\usepackage{subcaption}
\usepackage{mathtools}
\usepackage{amssymb}

\usepackage{amsmath,amsfonts,bm}

\def\eqref#1{equation~\ref{#1}}

\def\1{\bm{1}}

\DeclareMathAlphabet{\mathsfit}{\encodingdefault}{\sfdefault}{m}{sl}
\SetMathAlphabet{\mathsfit}{bold}{\encodingdefault}{\sfdefault}{bx}{n}

\def\gL{{\mathcal{L}}}

\def\sN{{\mathbb{N}}}

\def\sP{{\mathbb{P}}}

\def\sR{{\mathbb{R}}}

\newcommand{\E}{\mathbb{E}}

\newcommand{\Var}{\mathrm{Var}}

\def\cN{\mathcal{N}}
\def\I{\mathbb{I}}

\def\V{\mathbb{V}}
\newcommand\norm[1]{\left\lVert#1\right\rVert}

\def\cN{\mathcal{N}}
\def\I{\mathbb{I}}

\def\V{\mathbb{V}}

\title{DP-AdamBC: Your DP-Adam Is Actually DP-SGD \\ (Unless You Apply Bias Correction)}
\author{
    Qiaoyue Tang, Frederick Shpilevskiy, Mathias L\'ecuyer
}
\affiliations{
    University of British Columbia \\

    Vancouver, British Columbia, Canada\\
    \{qiaoyuet, fshpil\}@cs.ubc.ca, mathias.lecuyer@ubc.ca
}

\begin{document}

\maketitle

\begin{abstract}
    The Adam optimizer is a popular choice in contemporary deep learning, due to its strong empirical performance. However we observe that in privacy sensitive scenarios, the traditional use of Differential Privacy (DP) with the Adam optimizer leads to sub-optimal performance on several tasks. We find that this performance degradation is due to a DP bias in Adam's second moment estimator, introduced by the addition of independent noise in the gradient computation to enforce DP guarantees.
    This DP bias leads to a different scaling for low variance parameter updates, that is inconsistent with the behavior of non-private Adam. 
    We propose DP-AdamBC, an optimization algorithm which removes the bias in the second moment estimation and retrieves the expected behaviour of Adam.
    Empirically, DP-AdamBC significantly improves the optimization performance of DP-Adam by up to 3.5\% in final accuracy in image, text, and graph node classification tasks.
\end{abstract}

\section{Introduction}
\label{sec:intro}
The Adam optimization algorithm \citep{orig_adam} is the default optimizer for several deep learning architectures and tasks, notably in Natural Language Processing (NLP), for which Stochastic Gradient Descent (SGD) tends to struggle. Even in vision tasks where Adam is less prevalent, it typically requires less parameter tuning than SGD to reach good performance.

On all these tasks, deep learning models can leak information about their training set \citep{carlini2019secret,carlini2021extracting,carlini2022membership,balle2022reconstructing}.
We consider settings in which the deep learning model's training data is privacy sensitive, and models are trained with Differential Privacy \citep{dwork2006calibrating,abadi2016deep} to provably prevent training example information leakage \citep{wasserman2010statistical}.
Intuitively, training DP models requires computing each minibatch gradient with DP guarantees by clipping per-example gradients and adding Gaussian noise (\S\ref{sec:motiv}), to bound the maximal influence of any data-point on the final model.
The DP gradients can then feed into any optimization algorithm without modification to update the model's parameters.
Due to its success in the non-private setting, Adam is also prevalent when training  DP models, for NLP \citep{li2021} and GNN \citep{daigavane2022nodelevel} models.
However we observe that when combined with DP, Adam does not perform as well as without privacy constraints: Adam suffers a larger degradation of performance compared to SGD on vision tasks, while NLP models perform poorly when training from scratch.

To understand this effect, we go back to the original intuition behind Adam \citep{orig_adam} that relies on exponential moving averages estimating the first and second moments of mini-batch gradients. We show that while DP noise does not affect the first moment, it does add a constant bias to the second.
Drawing on a recent empirical investigation that suggests that the performance of Adam may be linked to its update rule performing a smooth version of the sign descent update \citep{kunstner2023heavytailed}, we show that the additive shift in Adam's second moment estimate caused by DP noise moves the Adam update away from that of sign descent, by scaling the gradient dimensions with different magnitudes differently.
Indeed, under typical DP parameters, the DP bias added to the second moment estimates of DP-Adam dominate the second moment estimate, and makes DP-Adam a rescaled version of DP-SGD with momentum.
We show how to correct this DP noise induced bias, yielding a variation that we call DP-AdamBC.
Empirically, correcting Adam's second moment estimate for DP noise significantly increases test performance for Adam with DP, on tasks for which Adam is well suited.

We make the following contributions:
\begin{enumerate}
    \item We analyze the interaction between DP and the Adam optimizer, and show that DP noise introduces bias in Adam's second moment estimator (\S\ref{sec:motiv-bias}). We show theoretically, and verify empirically, that under typical DP parameters DP-Adam reduces to DP-SGD with momentum (\S\ref{sec:motiv-sgdm}). This behavior violates the sign-descent hypothesis for Adam's performance.
    \item We propose DP-AdamBC, a variation of DP-Adam that corrects for the bias introduced by DP noise. We show that DP-AdamBC is a consistent estimator for the Adam update, under the same simplifying assumptions that justify Adam's update. (\S\ref{sec:method}). 
    \item We empirically evaluate the effect of DP-AdamBC, and show that it yields significant improvements (up to $3.5$ percentage points of test accuracy) over DP-Adam. (\S\ref{sec:exp}).
\end{enumerate}

\noindent Our implementation is available at: \url{https://github.com/ubc-systopia/DP-AdamBC}. All Appendixes referenced in the paper are available in the long version of the paper \cite{tang2023dpadambc}.

\section{Adam and the Sign-descent Hypothesis}
\label{sec:background}
The Adam update \citep{orig_adam} is defined as follows. Denote the average gradient over a mini-batch of size B with respect to loss function $f$ at step $t$ as:
\[
    g_t = (1/B)\nabla f(\theta_{t-1})
\]
Let $\beta_1$ and $\beta_2$ be Adam's decay coefficients. At each step, Adam updates two estimators:
\[
    m_{t} \leftarrow \beta_{1} m_{t-1} + \left(1-\beta_{1}\right) g_t \ ; \ \ \ \widehat{m}_{t} \leftarrow m_{t} /\left(1-\beta_{1}^{t}\right) ,
\]
\[
    v_{t} \leftarrow \beta_{2} v_{t-1}+\left(1-\beta_{2}\right) g_t^{2} \ ; \ \ \ \widehat{v}_{t} \leftarrow v_{t} /\left(1-\beta_{2}^{t}\right) .
\]
Finally, the Adam update for the model's parameters is:
\[
    \theta_t \leftarrow \theta_{t-1} - \eta \Delta_t \ ; \ \ \ \Delta_t = \hat{m}_t / (\sqrt{\hat{v}_t} + \gamma) ,
\]
with learning rate $\eta$, and $\gamma>0$ a small numerical stability constant.
Intuitively, Adam's $\hat{m}_t$ and $\hat{v}_t$ use an exponential moving average to estimate $\E[g_t]$ and $\E[g_t^2]$, the vector of first and second moment of each parameter's gradient, respectively. The final update is thus approximating $\E[g_t]/\sqrt{\E[g_t^2]}$.

The reasons for Adam's performance are not fully understood. However, recent evidence \citep{kunstner2023heavytailed} supports the hypothesis that Adam derives its empirical performance from being a smoothed out version of sign descent. At a high level, Adam performs well in settings (e.g., NLP) where sign descent also performs well, at least when running with full (or very large) batch.
We next describe Adam's update rule under this sign descent hypothesis, before working out the impact of DP noise on this interpretation. Let $\E{}$ and $\V{}$ denotes the expectation and variance respectively,
\begin{enumerate}
\item If for parameter $i$, $|\E{[g_t]}|_i \gg \sqrt{\V{[g_t]}_i}$, then the update's direction is clear. And since $|\E{[g_t]}|_i \approx \sqrt{\E{[g^2_t]}_i}$, the Adam update is $\E{[g_t]}_i / \sqrt{\E{[g^2_t]}_i} \approx \pm 1$, and Adam is sign descent. Updates are {\em not scaled based on $|\E{[g_t]}|_i$} as in SGD.
\item If for parameter $i$, $|\E{[g_t]}|_i \not\gg \sqrt{\V{[g_t]}_i}$, the sign is less clear and Adam's update is in $[-1, 1]$, scaled closer to $0$ the more uncertain the sign is (smoothing behavior).
\end{enumerate}
Finally, Adam ensures numerical stability when $|\E{[g_t]}|_i \approx 0$ and $\V{[g_t]}_i \approx 0$ using the additive constant $\gamma$ in the denominator of the update. In that case, the update is approximately $\E{[g_t]}_i/\gamma \approx 0$.

To summarize, under the sign descent hypothesis, Adam  updates parameters with low variance gradients using a constant size $\pm 1$ update (or $\pm \eta$ after the learning rate is applied), and rescales the update of parameters with high variance gradients towards $0$. As we describe next, adding DP to gradient computations breaks this interpretation of Adam as sign descent.

\section{Adam Update under Differential Privacy}
\label{sec:motiv}
Most optimization approaches for deep learning models with Differential Privacy follow a common recipe \cite{abadi2016deep}: compute each gradient update over a mini-batch with DP, and leverage DP's post-processing guarantee and composition properties to analyse the whole training procedure.
Computing a DP update over a mini-batch involves clipping per-example gradients to control the update's sensitivity, and adding {\em independent} Gaussian noise to the aggregated gradients.
Formally, for each step $t$, let $g_{n} = \nabla f(\theta_t, x_n)$ be the gradient for sample $n$, and let $C$, $\sigma$ be the maximum $L2$-norm clipping value and the noise multiplier, respectively. Given a mini-batch $B$, the DP gradient is:
\begin{gather*}
    \Tilde{g}_t = \overline{g}_t + (1/B) z_t \ ; \ \ \ z_t \sim \cN(0, \sigma^{2}C^{2}\I^{d}) \ ; \\
    \overline{g}_t = (1/B)\sum_{n} g_{n} / {\max{(1, \lVert g_n \rVert_{2}/C})},
\end{gather*}
where $\overline{g}_t$ is the mean of clipped gradients over the minibatch---a biased estimate of $g_t$---and $\Tilde{g}_t$ the DP gradient.

With this recipe, any optimizer that only takes mini-batch updates as input, such as Adam, can be applied to the DP update $\Tilde{g}$ and preserve privacy. This is how existing DP approaches using Adam work (e.g., \cite{li2021}), yielding the following update: let the superscript $p$ denote private version of a quantity, then
\begin{gather*}
    m_{t}^{p} \leftarrow \beta_{1} m_{t-1}^{p} + \left(1-\beta_{1}\right) \Tilde{g}_t, \, \widehat{m}_{t}^{p} \leftarrow m_{t}^{p} /\left(1-\beta_{1}^{t}\right), \\
    v_{t}^{p} \leftarrow \beta_{2} v_{t-1}^{p}+\left(1-\beta_{2}\right) \Tilde{g}_t^{2}, \, \widehat{v}_{t}^{p} \leftarrow v^p_{t} /\left(1-\beta_{2}^{t}\right), \\
    \theta_t \leftarrow \theta_{t-1} - \eta \widehat{m}_{t}^{p} / (\sqrt{\widehat{v}_{t}^{p}} + \gamma).
\end{gather*}

We show next that this DP-Adam algorithm uses a biased estimator for the second moment. This bias dominates the scale of the denominator in Adam's update, thus breaking the sign descent behaviour of Adam (\S\ref{sec:motiv-bias}) and reducing DP-Adam to DP-SGD with momentum and a specific learning rate schedule (\S\ref{sec:motiv-sgdm}).

\subsection{DP noise biases second moment estimates, breaking the sign descent behavior}
\label{sec:motiv-bias}
Under DP, Adam estimates the first and second moments as $m_t^{p}$ and $v_t^{p}$, and rescaled versions $\hat{m}_t^{p}$ and $\hat{v}_t^{p}$, using $\Tilde{g}_t$ in order to preserve privacy. Since the noise added for DP is independent of the gradient update, there is no impact on the first moment in expectation:
\begin{align}
\label{eq:mt}
\begin{split}
    \E{[m_t^{p}]} &= \E{\bigg[ (1-\beta_{1}) \sum_{\tau=1}^{t} \beta_{1}^{t-\tau} \Tilde{g}_\tau \bigg]} \\
    &= (1-\beta_{1}) \sum_{\tau=1}^{t} \beta_{1}^{t-\tau} \bigg( \E{[ \overline{g}_\tau ]} +  \underbrace{\frac{1}{B}\E{[z_\tau]}}_{\text{0}} \bigg) \triangleq \E{[m_t^{c}]}.
\end{split}
\end{align}
However, $v_t^{p}$ is now a biased estimate of the second moment of the mini-batch's update $\bar{g}_t$, as it incurs a constant shift due to DP noise \citep{tang2023dpadam}. By independence of the DP noise $z_t$ and $\bar{g}_t$, we have that:
\begin{align}
\label{eq:vt}
\begin{split}
    \E{[v_t^{p}]} &= \E{\bigg[ (1-\beta_{2}) \sum_{\tau=1}^{t} \beta_{2}^{t-\tau} {\Tilde{g}_\tau}^{2} \bigg]} \\
    &= \underbrace{(1-\beta_{2}) \sum_{\tau=1}^{t} \beta_{2}^{t-\tau} \E{ [ \overline{g}_\tau^{2}]}}_{\triangleq \ \E{[v_t^{c}]}} + (1-\beta_{2}^{t})\underbrace{\bigg(\frac{\sigma C}{B} \bigg)^{2}}_{\Phi} .
\end{split}
\end{align}
In these equations, $\E{[m_t^{c}]}$ and $\E{[v_t^{c}]}$ are the quantities that would be estimated under regular Adam (without DP noise), computed with respect to $\bar{g}_t$ (clipped gradients for DP).

\begin{figure}[htb]
\centering
\includegraphics[width=0.90\linewidth]{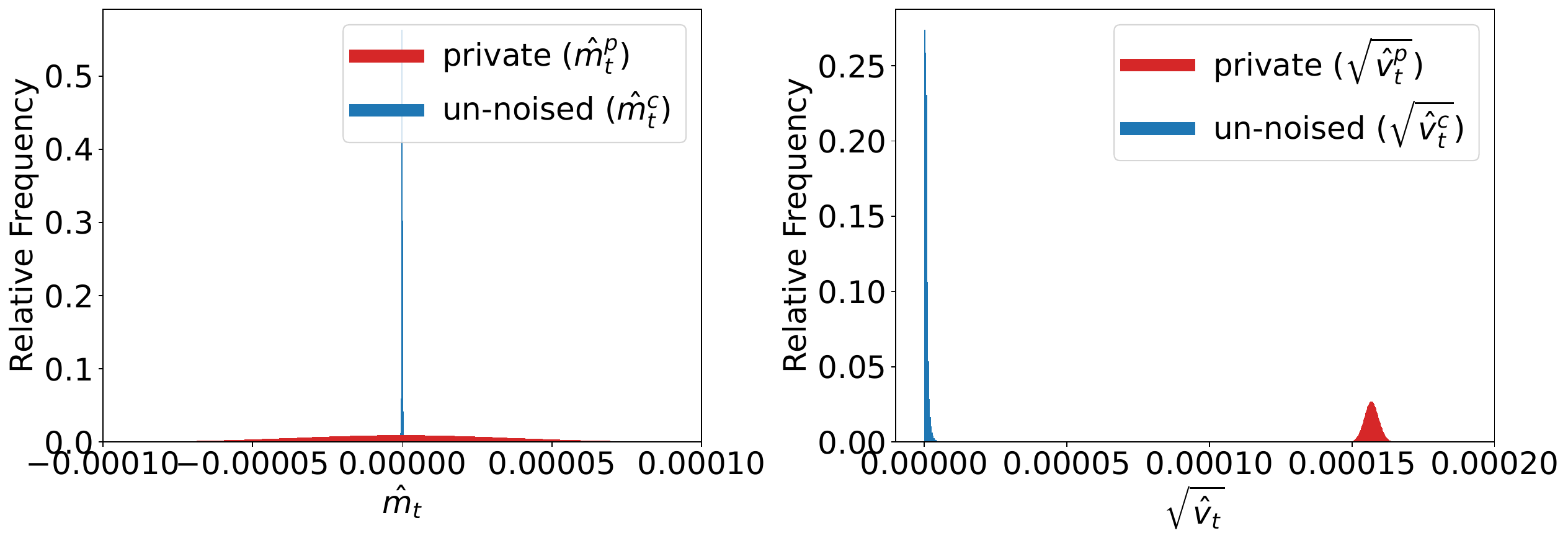}
\caption{Histogram of \textbf{Left:} un-noised ($\hat{m}_t^c$) and private ($\hat{m}_t^p$) first moment estimates, \textbf{Right:} un-noised ($\hat{v}_t^c$) and private ($\hat{v}_t^p$) second moment estimates near end of training at $t=10000$, using the SNLI dataset with $B=256, C=0.1, \sigma=0.4, \beta_2=0.999$, $\Phi \approx \textnormal{2.441e-8}$ for large $t$.}
\label{fig:motiv_10001}
\end{figure}

We use a text classification dataset (SNLI) to demonstrate the effect of DP noise on first and second moment estimates, with $B=256, C=0.1, \sigma=0.4, \beta_2=0.999$, $\Phi = \textnormal{2.441e-8}$ with large $t$.

Figure \ref{fig:motiv_10001} (Left) shows the histogram of values of the first moment estimates $\hat{m}_t^c$ (clipped gradients, no noise) for each dimension, and private $\hat{m}_t^p$ (clipped and noised gradients), at the end of training. We observe that the center of the distributions align, suggesting that $\E{[\hat{m}_t^p]} = \E{[\hat{m}_t^c]}$ as in Equation \ref{eq:mt}. The private first moment distribution has larger variance compared to the clean distribution as a result of DP noise. Figure \ref{fig:motiv_10001} (Right) shows the histogram of $\hat{v}_t^c$ (clipped, no noise) and private $\hat{v}_t^p$ (clipped and noised) second moment estimates at the end of training. We see that the distributions of $\hat{v}_t^c$ and $\hat{v}_t^p$ are quite different, with a shift in the center approximately equal to $\sqrt{\Phi}$. This suggests that the DP noise variance dominates the scale of $\hat{v}_t^p$ in Equation \ref{eq:vt}.

To understand the implication of DP noise bias $\Phi$, let us follow the original Adam paper \cite{orig_adam} and interpret the update under the following assumption:

\begin{assumption}[Stationarity]\label{assumption:stationarity}
For all $\tau$ in $[0, t]$, the (full) gradient is constant, $\nabla f(\theta_\tau) \triangleq \nabla f$, and minibatch gradients are i.i.d samples such that $\E[g_t] = \nabla f$.
\end{assumption}

\begin{remark}
    Note that Assumption \ref{assumption:stationarity} is not required for convergence (see Appendix F), nor is it used in empirical experiments. It is useful though, to reason about the behavior of DP-Adam and compare it to the intended behavior of Adam without DP, as we do next. The same assumption was used in Adam's original work for the same purpose, to reason about the quality of Adam's moment estimates [\cite{orig_adam}, \S3].
\end{remark}

Under Assumption \ref{assumption:stationarity}, with $\beta_1 \rightarrow 1$, $\beta_2$ such that $(1-\beta_1^{t})/{\sqrt{1-\beta_2^{t}}} = 1$, and for large enough $t$, we have that $\E{[\hat{m}_t^{c}]} \approx \E{[\bar{g}_t]}$ and $\E{[\hat{v}_t^{c}]} \approx \E{[\bar{g}_t^{2}]}$, and $\Delta_t =
\E{[\tilde{g}_t]} / \sqrt{\E{[\tilde{g}_t^2]}} = \E{[\bar{g}_t]} / \sqrt{\V{[\bar{g}_t]+\E(\bar{g}_t)^2}+\Phi}$.
Due to the extra DP bias $\Phi$ in the denominator of Adam's estimator, DP-Adam no longer follows the sign descent interpretation seen in \S\ref{sec:background}.

Focusing on the sign descent regime---when a parameter $i$ in the model has a large signal and small variance, such that $|\E{[\bar{g}_t]}|_i \approx \sqrt{\E{[\bar{g}^2_t]}_i}$---the Adam update becomes $\pm (|\E{[\bar{g}_t]}|_i / \sqrt{\E{[\bar{g}^2_t]}_{i} + \Phi})$ instead of $\pm 1$.
For example: if $|\E[\bar{g}_t]|_i=\sqrt{0.1\Phi}$, the update will be $\approx \pm 0.1$, whereas it will be $\approx \pm 1$ if $|\E[\bar{g}_t]|_i=\sqrt{10\Phi}$. In each case, without DP noise Adam would result in a $\pm 1$ update.

Importantly, re-scaling the learning rate $\eta$ is not sufficient to correct for this effect. Indeed, consider two parameters of the model indexed by $i$ and $j$ that, at step $t$, both have updates of small variance but different magnitude, say $|\E[\bar{g}_t]|_i=\sqrt{0.1\Phi}$ and $|\E[\bar{g}_t]|_j=\sqrt{10\Phi}$.  Then the Adam update for $i$ will be $\approx \pm 0.1$ and that of $j$ $\approx \pm 1$, and no uniform learning rate change can enforce a behavior close to sign descent for both $i$ and $j$ in this step. Indeed, under typical DP parameters, DP-Adam is closer for DP-SGD with momentum, as we show next.

\subsection{DP-Adam is DP-SGD with momentum}
\label{sec:motiv-sgdm}
As we saw on Figure \ref{fig:motiv_10001}, under typical DP parameters the DP noise bias $\Phi$ dominates $\hat{v}^p_t$. That is, $\Phi \gg \E[\bar{v}^c_t]$, and we have $\Delta_t \approx \hat{m}^p_t / \sqrt{\Phi}$. Intuitively in this setting, the denominator of DP-Adam's update leads to a constant rescaling, instead of a sign descent \citep{kunstner2023heavytailed} or inverse variance conditioning \citep{balles2020dissecting}. Compensating by properly scaling the learning rate yields an update proportional to $m^p_t$, which is the update of DP-SGD with momentum.

More precisely, using the private gradients $\Tilde{g}_t$ in DP-SGD with Momentum (DP-SGDM) yields the following update:
\[
    b_t^{p} \leftarrow \beta b_{t-1}^{p} + \Tilde{g}_t \ ; \ \ \ \theta_t \leftarrow \theta_{t-1} + \eta b_t^{p} ,
\]
where $\beta \in [0, 1]$ is a momentum decay coefficient. Note the slightly different semantics for $\beta$ compared to Adam, as we follow the typical formulation of DP-SGDM.
We thus have $b_t^{p} = \sum_{\tau \leq t} \beta^{t-\tau} \Tilde{g}_t$ and $\hat{m}_t^{p} = \frac{1-\beta_1}{1-\beta_1^t} \sum_{\tau \leq t} \beta_1^{t-\tau} \Tilde{g}_t$.
Setting $\beta^{\textnormal{DP-SGDM}} = \beta_1^{\textnormal{DP-Adam}}$, and using the same updates $\Tilde{g}_t$ leads to $b_t^{p} = \frac{1-\beta}{1-\beta^t} \hat{m}_t^{p}$. In the DP regime where $\Phi \gg \E[\bar{v}^c_t]$, and thus $\hat{v}_t^{p} \approx \Phi$, the DP-Adam update is $\Delta_t \approx \hat{m}_t^{p} / \sqrt{\Phi} = \frac{1-\beta}{(1-\beta^t)\sqrt{\Phi}} b_t^{p}$. Hence, DP-Adam is DP-SGDM with the following learning rate schedule:
\begin{equation}
\label{eq:lr_convert}
    \eta^{\textnormal{DP-SGDM}} = \eta^{\textnormal{DP-Adam}} \bigg( \frac{1-\beta}{(1-\beta^t)\sqrt{\Phi}} \bigg).
\end{equation}

\begin{figure*}[t]
\centering
\includegraphics[width=0.7\textwidth]{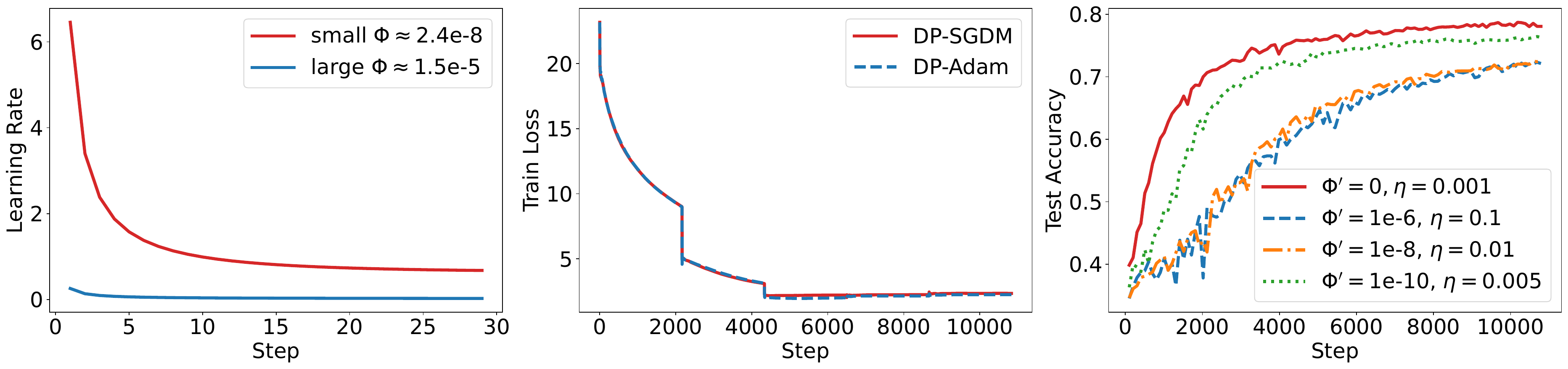}
\caption{DP-Adam behaves similarly to DP-SGDM with a specific learning rate (lr) schedule. \textbf{Left: }The implied lr schedule of DP-SGDM. \textbf{Middle: } DP-Adam and DP-SGDM with the specific lr schedule has similar training performance with a mean squared difference of 0.015 in training loss. \textbf{Right: }Performance degrades when adding a larger constant bias to un-noised (clipping-only) DP-Adam as it transitions to behave more like DP-SGDM.}
\label{fig:motiv_sgdm}
\end{figure*}

Figure \ref{fig:motiv_sgdm} empirically confirms this analysis in typical DP regimes.
Figure \ref{fig:motiv_sgdm} (Left) shows the learning rate schedule over step $t$ for the $\Phi$ values of two DP settings (`small' $\Phi$ with $B=256, C=0.1, \sigma=0.4$, `large' $\Phi$ with $B=256, C=1.0, \sigma=1.0$) when $\beta=0.9$ and $\eta^{\textnormal{DP-Adam}}=0.001$. We see that DP-Adam emulates DP-SGDM with an exponentially decreasing learning rate schedule, with an asymptotic value that depends on $\Phi$ ($\approx0.645$ for $\Phi\approx$2.4e-8, $\approx0.026$ for $\Phi\approx$1.5e-5).

Figure \ref{fig:motiv_sgdm} (Middle) shows the training loss over steps for DP-Adam ($B=256, C=0.1, \sigma=0.4, \eta=0.001$) and DP-SGDM (same $B, C, \sigma$, $\eta$ follows Eq. \ref{eq:lr_convert}, converging to $\approx 6.4$), on the SNLI dataset. We observe that the two algorithms have almost identical training performance:  their respective loss over steps closely aligns, with a mean squared difference of $\approx0.015$ over the entire training.

Figure \ref{fig:motiv_sgdm} (Right) shows the effect of adding a constant bias ($\Phi'$) to Adam's update denominator, without noise, on the SNLI dataset. That is, we update parameters with $\Delta_t = m_t^c/\sqrt{v_t^c + \Phi'}$, where $\Phi'=0$ implies un-noised DP-Adam (gradients are clipped, but no noise is added).
We tune $\eta$ for test accuracy at the end of training.
This experiment thus isolates the effect of second moment bias from DP noise.
We observe that on this text classification task, on which Adam performs better than SGD without DP, the performance of DP-Adam degrades as it transitions to DP-SGDM (more bias is added to the denominator). We conclude that DP-Adam's performance likely degrades due to the DP bias $\Phi$. Appendix D shows more performance comparisons between DP-SGDM and DP-Adam.

Prior work made similar observations on the effect of DP noise on DP-Adam's second moment estimator \cite{mohapatra2021role}. Their approach is to remove second moment scaling, which as we showed produces DP-SGDM. Instead, we show how to correct DP noise bias, yielding the DP-AdamBC variant that follows Adam's behavior without DP, despite the addition of noise.

\section{DP-Adam, Bias Corrected (DP-AdamBC)}
\label{sec:method}
Since we can compute the bias in $v_t^{p}$ due to DP noise (see Eq. (\ref{eq:vt})), we propose to correct for this bias by changing the Adam update $\Delta_t$ as follows:
\begin{equation}
\label{eq:adam_corr}
    \Delta_t = \frac{\hat{m}_t}{\sqrt{\max{\big(\hat{v}_t - \Phi}, \gamma' \big)}} = \frac{\hat{m}_t}{\sqrt{\max{\big(\hat{v}_t - (\sigma C / B)^{2}}, \gamma' \big)}}.
\end{equation}

\begin{algorithm}[h]
\DontPrintSemicolon
\SetAlgoLined
\scriptsize
\KwOut{Model parameters $\theta$}
\KwIn{Data $D=\{x_i\}_{i=1}^N$, $\eta$, $\sigma$, $B$, $C$, $\beta_1, \beta_2$, $\gamma'$, $\epsilon$-DP, $\delta$-DP; initialize $\theta_0$ randomly; $m_0 = 0, v_0 = 0$; total number of steps $T = f(\epsilon\textnormal{-DP}, \delta\textnormal{-DP}, B, N, \sigma)$}
 \For{$t = 1 \ldots, T$} {
    Take a random batch with sampling probability $B/N$ \;
    $g_i = \nabla \gL(\theta_{t-1}, x_i)$ \;
    $\Tilde{g}_t = \frac{1}{B} \big( \sum_{i} g_{i} / {\max{\big(1, \frac{\lVert g_i \rVert_{2}}{ C}} \big)} + z_t \big), z_t \sim \cN\big(0, \sigma^{2}C^{2}\I^{d}\big)$ \;
    $m_{t} \leftarrow \beta_{1} \cdot m_{t-1} + \left(1-\beta_{1}\right) \cdot \Tilde{g}_t$,\, $\widehat{m}_{t} \leftarrow m_{t} /\left(1-\beta_{1}^{t}\right)$ \;
    $v_{t} \leftarrow \beta_{2} \cdot v_{t-1}+\left(1-\beta_{2}\right) \cdot \Tilde{g}_t^{2}$,\, $\widehat{v}_{t} \leftarrow v_{t} /\left(1-\beta_{2}^{t}\right)$ \;
    $\theta_{t} \leftarrow \theta_{t-1} - \eta \cdot \hat{m}_t / \sqrt{\max{\big(\hat{v}_t - (\sigma C / B)^{2}}, \gamma' \big)}$
  }
 \caption{DP-AdamBC (with corrected DP bias in second moment estimation)}
 \label{alg:dp_adam}
\end{algorithm}

Algorithm \ref{alg:dp_adam} shows the overall DP-AdamBC optimization procedure, including the moment estimates from Adam. The main differences are the bias correction to the second moment estimate, and a different numerical stability constant, which we come back to later in this section, after discussing several important properties of DP-AdamBC.

\paragraph{Privacy Analysis.} Our bias corrected DP-AdamBC follows the same DP analysis as that of DP-Adam, and that of DP-SGD. Since both $\hat{m}_t$ and $\hat{v}_t$ are computed from the privatized gradient $\Tilde{g_t}$, the post-processing property of DP and composition over training iterations ensure privacy. The correction is based only on public parameters of DP-Adam: $\beta_{2}$, step $t$, batch size $B$, and the DP noise variance $(\sigma C)^{2}$. We prove the following proposition in Appendix E. In experiments (\S\ref{sec:exp}) we use R\'enyi DP for composition, though other techniques would also apply.

\begin{proposition}[Privacy guarantee of DP-AdamBC]
\label{prop:privacy_guarantee}
    Let the optimization algorithm DP-SGD($\theta, X, y, C, \sigma, B$) (Algorithm 1 in \citet{abadi2016deep}),
    with privacy analysis \textit{Compose}($T$, $\theta_{1,\ldots,T}$), be $(\epsilon, \delta)$-DP, then DP-AdamBC($\theta, X, y, C, \sigma, B$) with the same privacy analysis \textit{Compose}($T$, $\theta_{1,\ldots,T}$) is also $(\epsilon, \delta)$-DP.
\end{proposition}

\paragraph{Consistency of DP-AdamBC.}
Remember from \S\ref{sec:background} and \S\ref{sec:motiv-bias} that Adam seeks to approximate $\E[g_t]/\sqrt{\E[g_t^2]}$, and does under Stationarity (Assumption \ref{assumption:stationarity}).
Similarly, under Assumption \ref{assumption:stationarity}, DP-AdamBC is a consistent estimator of $\E[\bar{g}_t]/\sqrt{\E[\bar{g}_t^2]}$ as $\beta_1, \beta_2 \rightarrow 1$, and $t \rightarrow \infty$. Formally, calling $\hat{v}_t^{\textnormal{corr}} = \frac{(1-\beta_2)\sum_{\tau=1}^{t}\beta_2^{t-\tau}\Tilde{g}_t^2}{1-\beta_2^t} - \big( \frac{\sigma C}{B} \big)^2$, we have the following result, proven in Appendix B:
\begin{proposition}\label{proposition:consistent}
Under Assumption \ref{assumption:stationarity}, the DP-AdamBC update (without numerical stability constant) $\frac{\hat{m}_t^p}{\sqrt{\max(\hat{v}_t^{\textnormal{corr}}, 0)}}$ is a consistent estimator of $\frac{\E[\bar{g}_t]}{\sqrt{\E[\bar{g}_t^2]}}$ as $\beta_1, \beta_2 \rightarrow 1$, and $t \rightarrow \infty$.
\end{proposition}
Intuitively, under the stationarity assumption, DP-AdamBC estimates the Adam target update in the limit of averaging over a large number of steps. In practice, $\beta_1$ and $\beta_2$ trade-off the freshness of gradients used in the running estimates with the effect of averaging out DP noise. The DP-Adam update is not a consistent estimate of $\E[\bar{g}_t]/\sqrt{\E[\bar{g}_t^2]}$, but converges to $\E[\bar{g}_t]/\sqrt{\E[\bar{g}_t^2] + \Phi}$. Making $\Phi$ smaller would require increasing $B$ or decreasing $\sigma C$, resulting in a higher privacy cost per optimization step.

\paragraph{DP-AdamBC and sign-descent.} Thanks to its consistency property, the DP-AdamBC update on Equation \ref{eq:adam_corr} re-enables the sign descent interpretation for DP-Adam which closely tracks that of Adam. Ignoring the stochasticity introduced by measurements with DP noise for now:
\begin{enumerate}
    \item If for parameter $i$, $|\E{[\Bar{g}_t]}|_i \gg \sqrt{\V{[\Bar{g}_t]}_i + \Phi}$ , then $|\E{[\Bar{g}_t]}|_i \approx \sqrt{\E{[\Bar{g}^2_t]}_i}$, and $\Delta_t \approx \pm 1$. The update would be similar even without of our bias correction.
    \item If for parameter $i$, $|\E{[\Bar{g}_t]}|_i \gg \sqrt{\V{[\Bar{g}_t]}_i}$ but $|\E{[\Bar{g}_t]}|_i \ll \Phi$, then correcting for $\Phi$ ensures that $|\E{[\Bar{g}_t]}|_i \approx \sqrt{\E{[\Bar{g}^2_t]}_i}$, and $\Delta_t \approx \pm 1$, the expected behavior under Adam and the sign descent hypothesis. Without the correction, the update would be scaled as $\E{[\bar{g}_t]}/\Phi$ instead, and proportional to the gradient size, which is not the Adam or sign descent behavior.
    \item If for parameter $i$, $|\E{[\Bar{g}_t]}|_i \not\gg \sqrt{\V{[\Bar{g}_t]}_i + \Phi}$ (large gradient variance), $\Delta_t \in [-1, 1]$, performing a smooth (variance scaled) version of sign descent (not correcting for $\Phi$ would make the update closer to $0$, especially if $\Phi$ is large compared to $\V{[\Bar{g}_t]}_i$).
\end{enumerate}
In practice we cannot ignore the effect of DP noise of course. The first moment estimate $m_t^p$ is unbiased and adds variance to the optimization. We discuss the impact of stochastic measurements on the second moment next, while \S\ref{eval:empirical-effect-correction} details the empirical effects of our correction.

\paragraph{The numerical stability constant.}
The exponential moving average over DP quantities introduces measurement errors due to DP noise. It is thus possible that $\hat{v}_{i,t} - \Phi < \V{[\bar{g}_t]}_i$, and even that $\hat{v}_{i,t} - \Phi <0$. Our stability correction, $\textrm{max}( . , \gamma')$, deals with these cases similarly to Adam's $\gamma$. We expect that $\sqrt{\gamma'} \gg \gamma$ since the DP noise is typically larger than the gradients' variance.
To quantify this effect, we first analyze the error introduced by DP noise to $\hat{v}_t^{\textnormal{corr}}$ when considering a fixed sequence of clipped gradients. That is, the sequence of parameters $\theta_t$ and mini-batches is fixed. This measures the deviation of $\hat{v}_t^{\textnormal{corr}}$ from $\hat{v}_t^c$ due to DP noise, a measurement error from the quantity we are trying to estimate on a fixed sequence of parameters. In this case:

\begin{proposition}
\label{prop:bound_fixed}
Consider a fixed-in-advance sequence of model parameters $\theta_t$ and mini-batches.
For $0 < \alpha < 1$, for each dimension $i$, we have $\sP[\lvert
 \hat{v}_t^{\textnormal{corr}} - \hat{v}_t^c \rvert_i \geq \xi ] \leq \alpha$ with:
\begin{align*}
    \xi \geq
    \begin{cases}
        (\frac{1-\beta_2}{1-\beta_2^t}) \sqrt{\ln{(1/\frac{\alpha}{2})}(2v^2)}
        & 0 \leq \frac{\xi(1-\beta_2^t)}{1-\beta_2}  \leq \frac{\nu^2}{b} \\
        (\frac{1-\beta_2}{1-\beta_2^t}) \ln{(1/\frac{\alpha}{2})}2b
        & \frac{\xi(1-\beta_2^t)}{1-\beta_2}    \geq \frac{\nu^2}{b},
    \end{cases}
\end{align*}
where $\nu = (\frac{4\sigma^2C^2}{B^2})\sqrt{\frac{1-\beta_2^{2t}}{1-\beta_2^2}}, b=\frac{4\sigma^2C^2}{B^2}$.
\end{proposition}
The proof is in Appendix C.
For our SNLI example, this yields a bound of $5.933$e-09 at probability 0.05 at $t=10000$. We show in Appendix C, using empirical measurements, that this bound is accurate.
In practice, the values of $\hat{v}_t^{\textnormal{corr}}$ error are concentrated around their mean $\hat{v}_t^c$, with $\hat{v}_t^{\textnormal{corr}} - \hat{v}_t^c$ smaller than large values of $\hat{v}_t^c$, making bias correction practical.

While it can still happen that $|\Delta_{i, t}| \geq 1$, we show in \S\ref{sec:exp} that debiasing the second moment to follow the sign descent interpretation yields an improvement in model accuracy.
Finally, Appendix C also shows a Martingale analysis that does not assume a fixed sequence of parameters $\theta_t$, which are treated as random variables dependent on the noise at previous steps.
\begin{proposition}
\label{prop:bound_martingale}
    For $0 < \alpha < 1$, for each dimension $i$, we have $\sP[\lvert \hat{v}_t^p - \E{[\hat{v}_t^p]} \rvert_i \geq \xi] \leq \alpha$ with:
    \begin{align*}
    \xi \geq
    \begin{cases}
        (\frac{1-\beta_2}{1-\beta_2^t}) \sqrt{\ln{(1/\frac{\alpha}{2})}(2v^2)}
        & 0 \leq \frac{\xi(1-\beta_2^t)}{1-\beta_2}  \leq \frac{\nu^2}{b} \\
        (\frac{1-\beta_2}{1-\beta_2^t}) \ln{(1/\frac{\alpha}{2})}2b
        & \frac{\xi(1-\beta_2^t)}{1-\beta_2}  \geq \frac{\nu^2}{b},
    \end{cases}
\end{align*}
where $\nu = 2\sqrt{\frac{1-\beta_2^{2t}}{1-\beta_2^2}}( \frac{\sigma^2C^2}{B^2} + \frac{\sigma C^2}{B} )$, $b = \frac{4\sigma^2C^2}{B^2}$.
\end{proposition}
The error bound to $\E[\hat{v}_t^{\textnormal{corr}}]$ is much larger in this case, and not as useful in practice since we want to scale $\gamma'$ based on the realized trajectory.

\paragraph{Convergence of DP-AdamBC.}
To show Assumption \ref{assumption:stationarity} is not required for convergence, we study DP-AdamBC and DP-Adam under the setting of \citet{défossez2022simple}, adding the bounded gradient assumption from \citet{li2023dp2} to adapt it to the DP setting.
The main difference is we derive a high probability bound using techniques similar to that of Proposition \ref{prop:bound_martingale}. This allows us to deal with tecnically unbounded DP noise sampled from a Normal distribution. Note that both the theoretical convergence result and empirical results do not rely on Assumption \ref{assumption:stationarity}, which is only useful for matching the intuition to that of Adam and sign descent (and informs our algorithm). The detailed convergence rates and proofs, as well as a discussion, are in Appendix F.

\section{Empirical effect of Correcting for DP bias}
\label{sec:exp}

We compare the performance of DP-SGD, DP-Adam, and DP-AdamBC on image, text and graph node classification tasks with CIFAR10 \citep{cifar10_data}, SNLI \citep{bowman-etal-2015-large}, QNLI \citep{wang2019glue} and ogbn-arxiv \citep{hu2021open} datasets. We evaluate the training-from-scratch setting: for image classification, we use a 5-layer CNN model and all of the model parameters are initialized randomly; for text classification, only the last encoder and the classifier blocks are initialized randomly and the other layers inherit weights from pre-trained BERT-base model \citep{bert_paper}; for node classification, we train a DP-GCN model \citep{daigavane2022nodelevel} from scratch without per-layer clipping. For each optimizer, we tune the learning rate, as well as $\gamma$ or $\gamma'$, to maximize test accuracy at different values of $\epsilon$ for $\delta=1$e-5: $\epsilon \in \{1, 3, 7\}$ for CIFAR10, SNLI and QNLI, and $\epsilon \in \{3, 6, 12\}$ for ogbn-arxiv.
Appendix A includes the detailed dataset and model information, experiment setups and hyperparameters.

Table \ref{tab:res_multi_eps} shows the performance of different optimizers. DP-AdamBC often outperforms both DP-Adam and DP-SGD on NLP datasets (SNLI and QNLI), generally by 1 percentage point and up to 3.5 percentage points on SNLI for large $\epsilon=7$. DP-AdamBC retains a similar performance to DP-Adam on CIFAR10 while DP-SGD outperforms both, and even has an advantage over both DP-Adam and DP-SGD on obgn-arxiv for smaller $\epsilon$ values (4 percentage point at $\epsilon=3$, and 1.5 at $\epsilon=6$).
In Appendix D, we include full training trajectory plots (Figure 8), graphical comparison of optimizers' performances (Figure 6), and further examine the generalizability of our method by comparing to baselines with larger dataset and models (Figure 9).

{\bf Discussion.} Based on the experiment results and Adam's sign descent behaviour \citep{kunstner2023heavytailed}, we hypothesize that DP-AdamBC has a larger advantage on tasks and architectures for which Adam and sign descent outperform SGD in the non-private case. The hypothesis follows from DP-Adam's similarity to DP-SGD-with-Momentum (\S\ref{sec:motiv}), showing that DP-SGD and DP-Adam are closer to SGD-style algorithms, whereas DP-AdamBC is closer to the intended behavior of Adam under DP. Our experiments provide some evidence to support this reasoning: DP-AdamBC outperforms other approaches on tasks where Adam outperforms in the non-private case (WikiText-2 Transformer-XL experiment, Figure 3 \citet{kunstner2023heavytailed}); in the two cases in which DP-SGD or DP-Adam perform similarly to DP-AdamBC (CIFAR10 and obgn-arxiv in Figure 3), SGD is well documented to perform better without privacy (\citet{wilson2018marginal}, Table 1 in \citet{daigavane2022nodelevel}, respectively). Therefore, we would recommend using DP-AdamBC for DP training on tasks and model architectures on which Adam is expected (or has often been documented) to perform better than SGD without privacy. This includes modern NLP tasks with transformer-based models where Adam has been used extensively for its strong empirical performances.

{\bf Comparisons to previous work.} We compare the performance of DP-AdamBC to that of a recent Adam-like adaptive optimizer specially developped for DP, named $\textnormal{DP}^{2}$ \citep{lidp2}. $\textnormal{DP}^{2}$ uses delayed pre-conditioners to better realize the benefits of adaptivity. However, the algorithm was only evaluated on simple models, and we show that it doesn not work on the deep learning models we consider. Figure \ref{fig:compare_dp2_snli_only} (Left) shows the comparison between $\textnormal{DP}^{2}$-
RMSProp, DP-AdamBC and DP-SGD on CIFAR10 on SNLI dataset with Bert-base model. We observe that $\textnormal{DP}^{2}$-RMSProp first follows DP-SGD (since the first steps use this optimizer), and then struggles to converge on deep learning tasks, leading to poor performance. Indeed switching between two optimizers seems to make $\textnormal{DP}^{2}$ unstable: Figure \ref{fig:compare_dp2_snli_only} (Right) shows the performance of $\textnormal{DP}^{2}$ with different $s$ (switching frequency). We observe that during training, $\textnormal{DP}^{2}$'s performance either has large turbulence or drops significantly at switching points between optimizers. DP-AdamBC does not suffer from this issue. More analysis and experiments are in Appendix D.

\begin{table}[htb]
\centering
\resizebox{0.44\textwidth}{!}{%
\begin{tabular}{ccccc}
\hline\hline
 &  & $\epsilon \approx 1$ & $\epsilon \approx 3$ & $\epsilon \approx 7$ \\ \hline
\multirow{3}{*}{SNLI} & DP-SGD & \textbf{48.03 (1.25)} & 45.11 (1.84) & 51.04 (0.52) \\
 & DP-Adam & 44.72 (1.26) & 47.52 (1.75) & 52.63 (1.91) \\
 & DP-AdamBC & 45.17 (1.04) & \textbf{50.08 (1.57)} & \textbf{56.08 (0.99)} \\ \hline
\multirow{3}{*}{QNLI} & DP-SGD & 57.10 (1.59) & 58.85 (1.20) & 58.29 (0.92) \\
 & DP-Adam & 58.00 (2.05) & 60.72 (1.12) & 61.23 (1.30) \\
 & DP-AdamBC & \textbf{58.32 (1.90)} & \textbf{61.42 (0.99)} & \textbf{62.83 (1.60)} \\ \hline
\multirow{3}{*}{CIFAR10} & DP-SGD & \textbf{52.37 (0.50)} & \textbf{57.30 (0.76)} & \textbf{65.30 (0.33)} \\
 & DP-Adam & 51.89 (0.69) & 54.08 (0.41) & 62.24 (0.10) \\
 & DP-AdamBC & 49.75 (0.56) & 54.27 (0.23) & 63.43 (0.43) \\ \hline
 &  & $\epsilon \approx 3$ & $\epsilon \approx 6$ & $\epsilon \approx 12$ \\ \hline
\multirow{3}{*}{obgn-arxiv} & DP-SGD & 45.35 (1.38) & 49.12 (1.90) & \textbf{54.20 (0.62)} \\
 & DP-Adam & 46.55 (0.54) & 51.98 (0.48) & 54.02 (0.18) \\
 & DP-AdamBC & \textbf{50.51 (0.56)} & \textbf{53.40 (0.28)} & 53.81 (0.34) \\ \hline\hline
\end{tabular}%
}
\caption{Accuracy under different optimizers, for several privacy budgets. Hyper-parameters are tuned for each target $\epsilon$ and optimizer. Mean (standard deviation) over 5 runs for the best hyper-parameters.}
\label{tab:res_multi_eps}
\end{table}

\begin{figure}[t]
\centering
\includegraphics[width=0.85\linewidth]{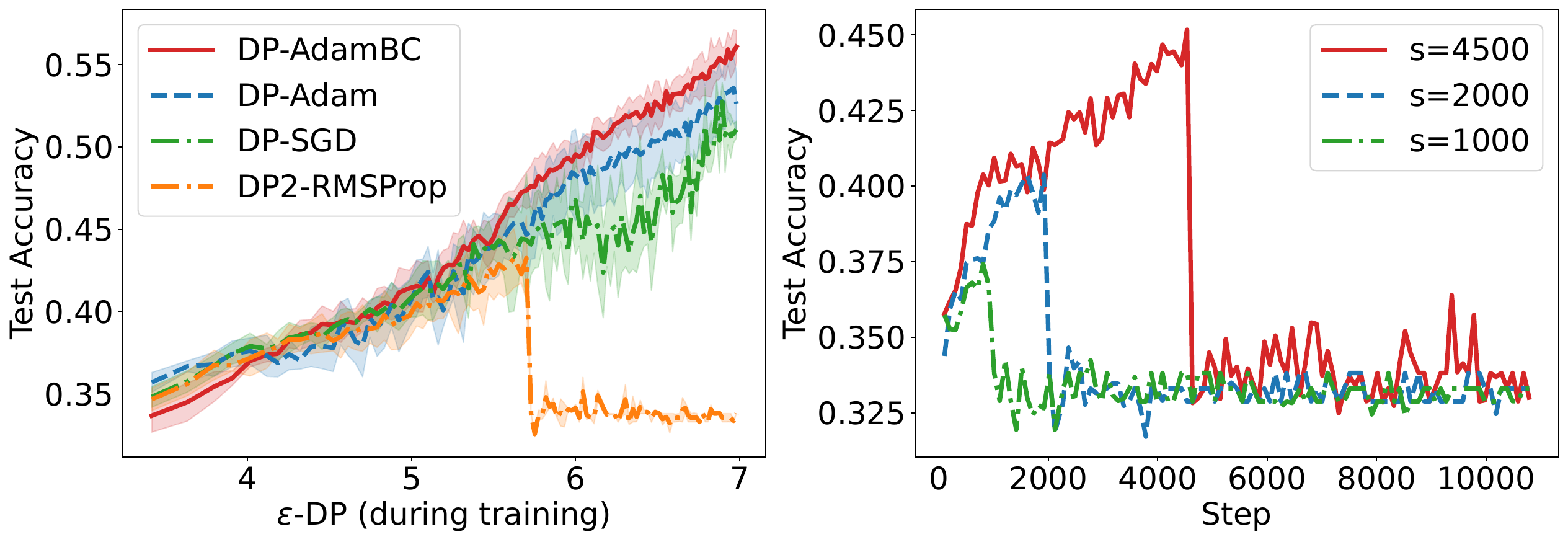}
\caption{\textbf{Left:} Comparison between DP2RMSProp, DP-AdamBC, DP-Adam and DP-SGD and \textbf{Right:} the performance of DP2RMSProp with different phase switching frequency $s$ on SNLI with Bert-base.}
\label{fig:compare_dp2_snli_only}
\end{figure}

\subsection{Empirical Effect of Bias Correction}
\label{eval:empirical-effect-correction}

\paragraph{First and second moment estimates of un-noised and private gradients.}
We numerically compare the scale of the first and second moment estimates based on un-noised and private gradients, $\hat{m}_t^c, \hat{m}_t^p, \hat{v}_t^c, \hat{v}_t^p$ respectively, at different training step $t$. The corresponding un-noised, noised and corrected updates are $\Delta_t^{c} = \frac{\hat{m}_t^p}{\sqrt{\hat{v}_t^c}}, \Delta_t^{p} = \frac{\hat{m}_t^p}{\sqrt{\hat{v}_t^p}}$ and $\Delta_t^{\textnormal{corr}} = \frac{\hat{m}_t^p}{\sqrt{\hat{v}_t^\textnormal{corr}}}$.
Table \ref{tab:exp2} shows the summary statistics of these variables near end of training, computed with the SNLI dataset with $B=256, C=0.1, \sigma=0.4, \Phi \approx 2.441$e-8 in the limit of $t$.
We observe that the difference between $\hat{m}_t^{c}$ and $\hat{m}_t^{p}$ is much smaller than that of $\hat{v}_t^{c}$ and $\hat{v}_t^{p}$, especially in the mean values (the empirical measures of the expectation). In particular, the mean of $\hat{v}_t^{p}$ is approximately $\Phi$, which suggests that the DP bias $\Phi$ dominates over the un-noised estimates of second moment $\hat{v}_t^{c}$.
We also observe that the scale of $\hat{v}_t^{p}$ is generally close to $\Phi$, which suggests the private estimate of the second moments are largely affected by the DP noise.
The scale of the corrected second moment estimates, $\hat{v}_t^{\textnormal{corr}} = \max{(\hat{v}_t - \Phi , \gamma')}$ is closer to the scale of $\hat{v}_t^{c}$, with the numerical stability constant ($\gamma'=3$e-10) preventing tiny denominator values.
If no correction is imposed, $\Phi$ dominates in $\E{[\tilde{g}_t]}$ making the update smaller. The tuned learning rate is larger to compensate, but the update $\Delta_t$ is still proportional to the first moment $\E[\bar{g}_t]$. This is not compatible with the behavior of sign descent (\S\ref{sec:method}).
\begin{table}[t]
\centering
\renewcommand{\arraystretch}{1.1}
\resizebox{0.45\textwidth}{!}{%
\begin{tabular}{c|cccccc}
\hline
 & \textbf{Min} & \textbf{Q1} & \textbf{Median} & \textbf{Q3} & \textbf{Max} & \textbf{Mean} \\ \hline
$m_t^c$ & -7.505e-05 & -7.051e-08 & -2.170e-18 & 7.056e-08 & 7.516e-05 & 4.194e-10 \\
$\hat{m}_t^p$ & -1.879e-04 & -2.428e-05 & 1.204e-08 & 2.427e-05 & 1.833e-04 & 6.120e-09 \\ \hline
$\hat{v}_t^{c}$ & 4.119e-24 & 8.297e-14 & 4.090e-13 & 9.819e-13 & 2.729e-08 & 4.032e-12 \\
$\hat{v}_t^{p}$ & 2.068e-08 & 2.408e-08 & 2.460e-08 & 2.513e-08 & 5.524e-08 & 2.461e-08 \\
$\hat{v}_t^{\textnormal{corr}}$ & 3.000e-10 & 3.000e-10 & 3.000e-10 & 7.137e-10 & 3.082e-08 & 5.633e-10 \\ \hline
$\Delta_t^{c}$ & -2.732e+07 & -3.933e+01 & 1.507e-02 & 3.938e+01 & 1.938e+07 & -2.663e+01 \\
$\Delta_t^{p}$ & -1.218 & -1.548e-01 & 7.680e-05 & 1.548e-01 & 1.159 & 3.921e-05 \\
$\Delta_t^{\textnormal{corr}}$ & -1.085e+01 & -1.116 & 5.473e-04 & 1.116 & 1.026e+01 & 3.650e-04 \\ \hline
\end{tabular}%
}
\caption{Moment estimates with un-noised and noised gradient, w/ and w/o bias correction, at step $t=10000$.}
\label{tab:exp2}
\end{table}

To further study the effect of DP noise and of our bias correction, we compare the distribution of the private, un-noised, and corrected variables. The same dataset and hyperparameters are used for demonstration.
Figure \ref{fig:mt_vt_hist} (Left) and (Middle) shows the histogram of private ($\sqrt{\hat{v}_t^p}$), un-noised ($\sqrt{\hat{v}_t^c}$) and corrected ($\sqrt{\hat{v}_t^{\textnormal{corr}}}$) second moment estimates, when $\gamma'=3e$-12 and $3e$-10 respectively.
We see that the distributions of $\hat{v}_t^c$ and $\hat{v}_t^p$ are quite different, with a shift in the center approximately equal to $\sqrt{\Phi}$. This suggests that the DP noise variance dominates the scale of $v_t^p$ in Equation \ref{eq:vt}. The corrected second moment estimates are much closer in scale to the clean estimates, with the gap near 0 due to the effect of the numerical stability constant $\gamma'$.
Figure \ref{fig:mt_vt_hist} (Right) shows the distribution of the noised ($\Delta_t^p$) and corrected ($\Delta_t^{\textnormal{corr}}$) Adam updates with respect to the noised first moment $\hat{m}_t^{p}$, rescaled to $[-1, 1]$. We observe that the private distribution is heavily concentrated around 0. The bias correction alleviates the concentration around 0 in the distribution, which is consistent with the interpretation in \S \ref{sec:method}.

\begin{figure}[t]
\centering
\includegraphics[width=0.95\linewidth]{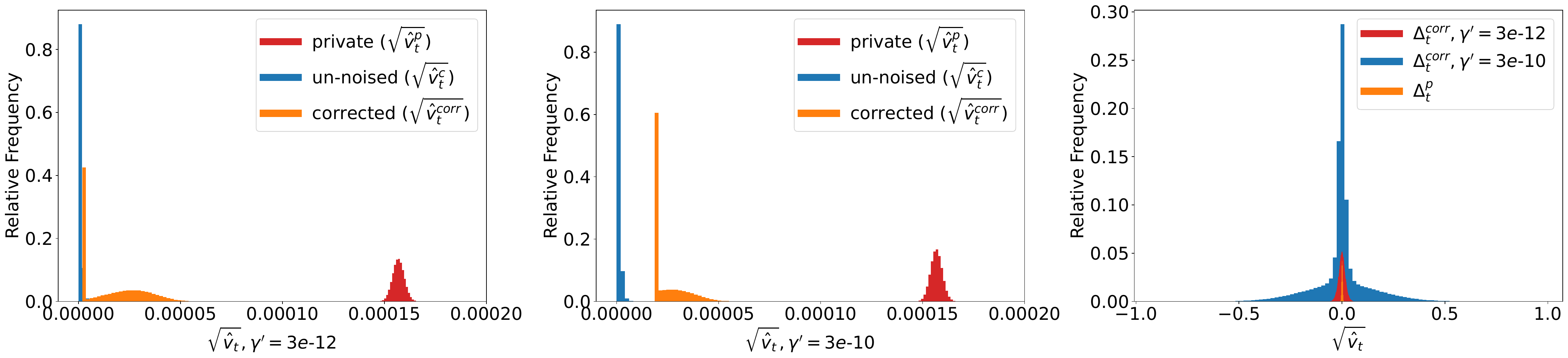}
\caption{Histogram of private ($\hat{v}_t^p$), un-noised ($\hat{v}_t^c$) and corrected ($\hat{v}_t^{\textnormal{corr}}$) second moment estimates with \textbf{Left: } $\gamma'=3e$-12, \textbf{Middle: }$\gamma'=3e$-10. \textbf{Right: } private ($\Delta_t^p$) and corrected ($\Delta_t^{\textnormal{corr}}$) Adam updates with respect to $m_t^{p}$.
}
\label{fig:mt_vt_hist}
\end{figure}

\paragraph{Correcting second moment with different values.}
We test whether the noise variance $\Phi$ is indeed the correct value to subtract from the noisily estimated $v_t^{p}$, by subtracting other values $\Phi'$ at different scales instead. In Figure \ref{fig:exp3_4} (Upper Left) we compare the performance of correcting $v_t^{p}$ with the true $\Phi$=2.4e-8 versus $\Phi'$. The experiments of DP-Adam($\Phi'$=1e-7) and DP-Adam($\Phi'$=1e-9) are trained using the same DP hyperparameters except changing value of $\Phi$ to $\Phi'$ and with coarsely tuned learning rates. We observe that both values of $\Phi'>\Phi$ or $\Phi'<\Phi$ lead to weaker performance. It suggests that the DP noise bias in the second moment estimate may be responsible for the degraded performance, and correcting for a different value does not provide a good estimate for $\E{[\overline{g}^2_t]}$.

\subsection{Hyperparameter Analysis}

\paragraph{Effect of the numerical stability constant.}
The numerical stability constant $\gamma$ is known to affect the performance of Adam in the non-private setting, and $\gamma$ is often tuned as a hyperparameter \citep{Reddi2019}. Following the same logic, we test the effect of $\gamma'$ and $\gamma$ on the performance of DP-AdamBC and DP-Adam. Figure \ref{fig:exp3_4} (Upper Right) shows that $\gamma'$ indeed impacts the performance of DP-Adam: values of $v_t^{p}$ are small, and changing $\gamma'$ can avoid magnifying a large number of parameters with tiny estimates of $v_t^{c}$.
Figure \ref{fig:exp3_4} (Lower Left) shows the effect of tuning $\gamma$ in DP-Adam. We observe that DP-AdamBC's numerical stability constant does have an impact on performance, but smaller than DP-Adam's equivalent. This is because the large scale of $\Phi$ makes estimates of $v_t^{p}$ relatively large and similar among parameters. We also observe that tuning $\gamma$ with DP-Adam is not a substitute for correcting for DP noise bias $\Phi$, and DP-AdamBC achieves higher accuracy.

\begin{figure}[htb]
\centering
\includegraphics[width=0.9\linewidth]{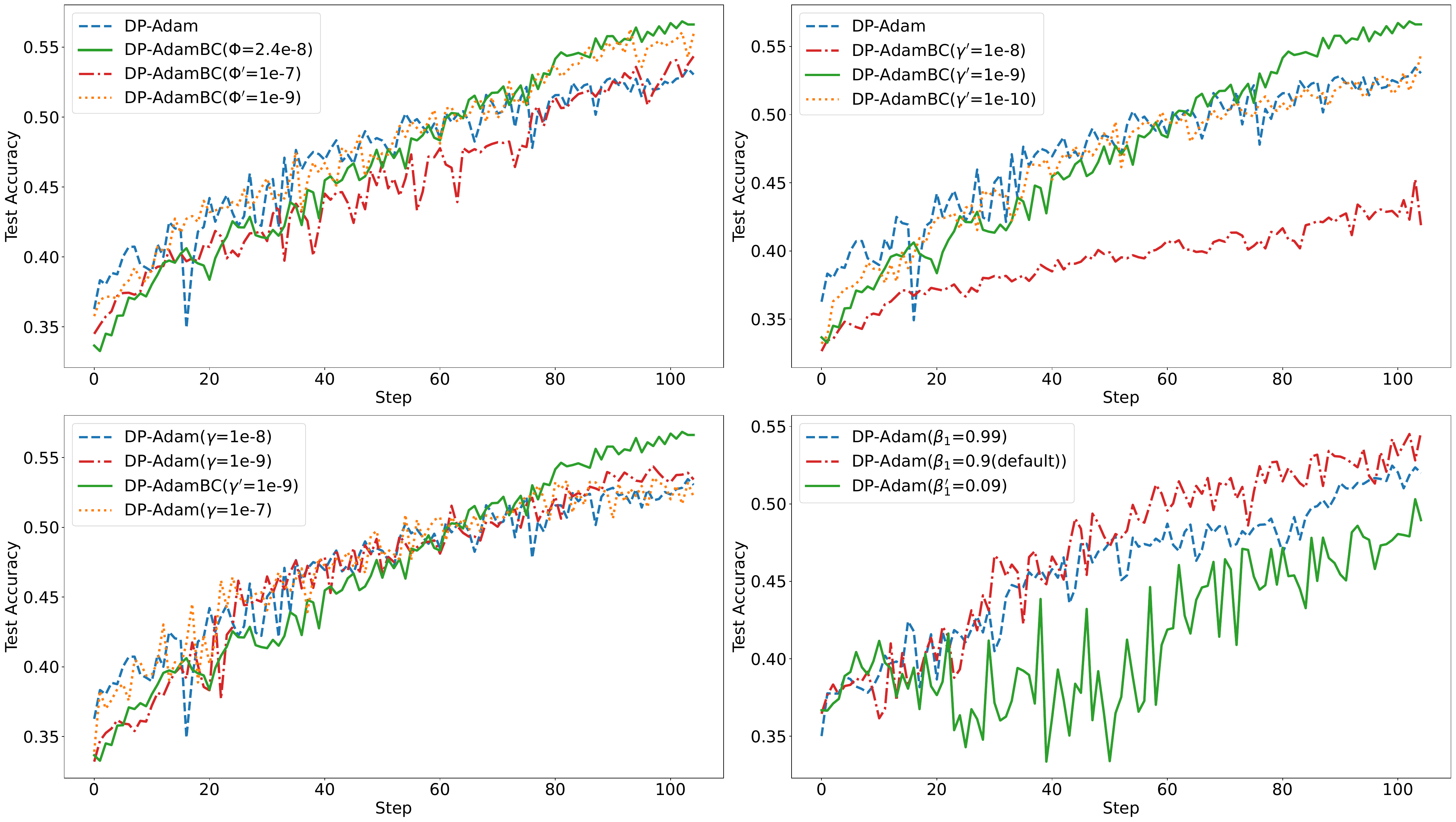}
\caption{Performance when \textbf{Upper Left: }subtracting different (fake) values of $\Phi$,
\textbf{Upper Right: }tuning $\gamma'$ in DP-AdamBC, \textbf{Lower Left: }tuning $\gamma$ in DP-Adam, \textbf{Lower Right: }tuning $\beta$s in DP-Adam. Tuning hyperparameters in DP-Adam cannot replace DP-AdamBC's bias correction.
}
\label{fig:exp3_4}
\end{figure}

\paragraph{Effect of the moving average coefficients.}
The $\beta$ coefficients control the effective length of the moving average window in Adam's estimates of the moments. It thus balances the effect of averaging out the noise, versus estimating moments with older gradients. A larger $\beta$ implies averaging over a longer sequence of past gradients, which potentially benefits performance by decreasing the effect of noise. Figure \ref{fig:exp3_4} (Lower Right) shows the effect of choosing different $\beta$ in DP-Adam, with the learning rate $\eta$ coarsely tuned from 1e-4 to 1e-2. As suggested in \citet{orig_adam}, we set $\beta_1$ and choose $\beta_2$ such that $(1-\beta_1) = \sqrt{1-\beta_2}$. We observe that setting $\beta$s too large or too small is worse than choosing the default values ($\beta_1=0.9, \beta_2=0.99$). Setting $\beta$ smaller shows a clear disadvantage as the performance is both worse and more volatile due to less smoothing over noise. Setting a larger $\beta$ results in similar performance at the end of training. However, lowering the effect of noise this way does not yield similar improvements as correcting for DP noise bias in the second moments.

\clearpage\clearpage

\section*{Acknowledgment}
We are grateful for the support of the Natural Sciences and Engineering Research Council of Canada (NSERC) [reference number RGPIN-2022-04469], as well as a Google Research Scholar award. This research was enabled by computational support provided by the Digital Research Alliance of Canada (alliancecan.ca), and by the University of British Columbia's Advanced Research Computing (UBC ARC).

\bibliography{main}

\clearpage

\appendix

\section{Experiment Setups}
\label{apdix:exp_setup}

\paragraph{Dataset.}
For image classification we use CIFAR10 \citep{cifar10_data} which has 50000 training images and 10000 test images. We use the standard train/test split and preprocessing steps as with \texttt{torchvision}.
For text classification we use the SNLI dataset \citep{bowman-etal-2015-large} and the QNLI dataset \citep{wang2019glue}. We use the same train/test split and preprocessing steps as in \texttt{Opacus}'s text classification with DP tutorial.
For node classification, we use the graph dataset, ogbn-arxiv \citep{hu2021open}. In this graph, nodes represent arXiv Computer Science papers and directed edges represents paper cites paper. This graph dataset has 169,343 nodes, average degree 13.7, 128 features, 40 classes, and 0.54/0.18/0.28 train/val/test split.

\paragraph{Model.}
For image classification on CIFAR10, we use a 5-layered CNN model as described in \citep{papernot2020tempered}.
For text classification on SNLI, we use a BERT-base model \citep{bert_paper} as in \texttt{Opacus}'s text classification tutorial.
For node classification, wee use a DP-GCN from \citet{daigavane2022nodelevel}. For DP-Adam and DP-AdamBC, this model has one encoder layer, one message passing layer, and two decoder layers. For DP-SGD, this model has two encoder layers and one decoder layer instead. In both cases, we use a latent size of 100 for the encoder, GNN, and decoder due to memory constraints.

\paragraph{Hyperparameters.}
For image classification on CIFAR10, the DP hyperparameters $C=1.0, \sigma=2.0$, batch size $=8192$, target $\epsilon = 7.1$ with R\'enyi DP for privacy accounting from Opacus \citep{opacus}. The learning rate for DP-AdamBC, DP-Adam and DP-SGD are 0.005, 0.007 and 2.5 respectively. The numerical stability constant is $\gamma'=5e$-8 and $\gamma=1$e-8 for DP-AdamBC, DP-Adam respectively. $\beta_1=0.9, \beta_2=0.999$ in both DP-AdamBC and DP-Adam. We use the Adam and SGD implementation from \texttt{optax} \citep{deepmind2020jax}.

For text classification on SNLI, the DP hyperparameters are $C=0.1, \sigma=0.4$, batch size $=256$, target $\epsilon = 7.0$ with R\'enyi DP for privacy accounting from Opacus \citep{opacus}. The learning rate for DP-AdamBC, DP-Adam and DP-SGD are 0.001, 0.01 and 45.0 respectively. The numerical stability constant is $\gamma'=1e$-10 and $\gamma=1$e-8 for DP-AdamBC, DP-Adam respectively. For text classification on QNLI, the DP hyperparameters are $C=0.1, \sigma=0.4$, $B=256$, target $\epsilon = 7.3$ with R\'enyi DP for privacy accounting from \texttt{Opacus}. The learning rate for DP-AdamBC, DP-Adam and DP-SGD are 0.003, 0.01 and 40.0 respectively. The numerical stability constant is $\gamma'=3e$-9 and $\gamma=1$e-8 for DP-AdamBC, DP-Adam respectively. We use the Adam and SGD implementation from \texttt{PyTorch} \citep{pytorch}.

For node classification in Figure \ref{fig:exp1}, we use a batch size of 10,000 and target $\epsilon = 12.0$ with Rényi DP privacy accounting from \citet{daigavane2022nodelevel}. We use noise multiplier $\lambda = 2$ and maximum degree $K = 7$ as they are defined in \citet{daigavane2022nodelevel}. In their paper, they use per-layer clipping; for each layer, their clipping thresholds are chosen as the 75th percentile of the gradient norms for that layer. We do not use per-layer clipping, we choose the clipping threshold $C$ as the median of their per-layer 75th percentile clipping thresholds. The learning rate for DP-AdamBC, DP-Adam, and DP-SGD are 0.003, 0.008, and 0.7 respectively. The numerical stability constant is $\gamma' = 2e$-6 and $\gamma = 1e$-12 for DP-AdamBC and DP-Adam respectively.

\paragraph{Hardware information.} We run experiments on local machine with Intel 11th 2.5GHz CPU and one Nvidia GeForce RTX 3090 GPU. The typical training time is about 15min, 2.5h and 10min on our image, text and node classification tasks respectively.

\section{Proof of Proposition \ref{proposition:consistent}}
\label{apdix:consistent}

\begin{proposition*}[\ref{proposition:consistent}]
Under Assumption \ref{assumption:stationarity}, the DP-AdamBC update (without numerical stability constant) $\frac{\hat{m}_t^p}{\sqrt{\max(\hat{v}_t^{\textnormal{corr}}, 0)}}$ is a consistent estimator of $\frac{\E[\bar{g}_t]}{\sqrt{\E[\bar{g}_t^2]}}$ as $\beta_1, \beta_2 \rightarrow 1$, and $t \rightarrow \infty$.
\end{proposition*}

\begin{proof}
Let $\hat{m}_t^{p}$, $\hat{v}_t^{p}$ and $\hat{v}_t^{\textnormal{corr}}$ be the following,
\begin{gather*}
    \hat{m}_t^{p} = \frac{(1-\beta_1)\sum_{\tau=1}^{t}\beta_1^{t-\tau}\Tilde{g}_t}{1-\beta_1^t}, \; \\
    \hat{v}_t^{p} = \frac{(1-\beta_2)\sum_{\tau=1}^{t}\beta_2^{t-\tau}\Tilde{g}_t^2}{1-\beta_2^t}, \; \\
    \hat{v}_t^{\textnormal{corr}} = \hat{v}_t^{p} - \big( \frac{\sigma C}{B} \big)^2,
\end{gather*}
We first show that $\hat{m}_t^{p} \xrightarrow{p} \E{[\Bar{g}_t]}$, and that $\hat{v}_t^{p} \nrightarrow \E{[\Bar{g}_t^2]}$ (though $\hat{v}_t^{\textnormal{corr}} \xrightarrow{p} \E{[\Bar{g}_t^2]}$), and study the full update at the end of the proof.

We start by showing that $\hat{m}_t^{p} \xrightarrow{p} \E{[\Bar{g}_t]}$. Using Chebychev's inequality, we have:
\begin{equation*}
    \sP\big( |\hat{m}_t^p - \E{[\Bar{g}_t]}| > \delta \big) \leq \frac{\E{[(\hat{m}_t^p - \E{[\Bar{g}_t])^2}]}}{\delta^2},
\end{equation*}
\begin{align*}
    \E{[(\hat{m}_t^p - \E{[\Bar{g}_t])^2}]} &= \V{[\hat{m}_t^p - \E{[\Bar{g}_t]}]} + \big(\E{[\hat{m}_t^p - \E{[\Bar{g}_t]}]}\big)^2 \\
    &= \big(\frac{1-\beta_1}{1-\beta_1^t}\big)^2 \sum_{\tau=1}^t \beta_1^{t-\tau} \big( \V{[\Bar{g}_\tau]} + \V{[z_\tau]} \big) \\
    &= \frac{(1-\beta_1)^2}{(1-\beta_1^t)} \big(\V{[\Bar{g}_t]} + \big(\frac{\sigma C}{B}\big)^2\big) \\
    &\leq \frac{(1-\beta_1)^2}{(1-\beta_1^t)} \big(4C^2 + \big(\frac{\sigma C}{B}\big)^2\big) \\
    &\rightarrow 0 \; \textnormal{when } t \rightarrow \infty, \beta_1 \rightarrow 1.
\end{align*}
The last inequality follows from the fact that $\Bar{g}_t$ is the clipped gradient (for DP), and hence $\V{[\Bar{g}_t]}$ is upper bounded by $4C^2$, the square of L2-norm clipping value.

Next we show that $\hat{v}_t^{p} \nrightarrow \E{[\Bar{g}_t^2]}$ but $\hat{v}_t^{\textnormal{corr}} \xrightarrow{p} \E{[\Bar{g}_t^2]}$. Using Chebychev's inequality again:
\begin{equation*}
\begin{gathered}
    \sP\big( |\hat{v}_t^p - \E{[\Bar{g}_t^2]}| > \delta \big) \leq \frac{\E{[(\hat{v}_t^p - \E{[\Bar{g}_t^2])^2]}}}{\delta^2}, \\
    \E{[(\hat{v}_t^p - \E{[\Bar{g}_t^2])^2]}} = \V{[\hat{v}_t^p - \E{[\Bar{g}_t^2]}]} + \big(\E{[\hat{v}_t^p - \E{[\Bar{g}_t^2]}]}\big)^2,
\end{gathered}
\end{equation*}
Moreover, $\V{[\hat{v}_t^p - \E{[\Bar{g}_t^2]}]} \rightarrow 0$ when $t\rightarrow \infty, \beta_2 \rightarrow 1$ since,
\begin{align*}
    &\V{[\hat{v}_t^p - \E{[\Bar{g}_t^2]}]} = \V{[\hat{v}_t^p]} \\
    &= \big( \frac{1-\beta_2}{1-\beta_2^t} \big)^2 \sum_{\tau=1}^{t} \beta_2^{2(t-\tau)} \V{[\Tilde{g}_{\tau}^2]} \\
    &= \big( \frac{1-\beta_2}{1-\beta_2^t} \big)^2 (\V{[\Bar{g}_t^2]} + \V{[z_t^2]} + 4\V{[\Bar{g}_t]}\V{[z_t]}) \sum_{\tau=1}^{t} \beta_2^{2(t-\tau)}  \\
    &= \big( \frac{1-\beta_2}{1-\beta_2^t} \big)^2  \big( \frac{1-\beta_2^{2t}}{1-\beta_2^2} \big) \bigg( \V{[\Bar{g}_t^2]} + 2\big( \frac{\sigma C}{B} \big)^2 + 4\V{[\Bar{g}_t]} \big( \frac{\sigma C}{B} \big)^2\bigg) \\
    &\rightarrow 0 \; \textnormal{as } t\rightarrow \infty, \beta_2 \rightarrow 1.
\end{align*}
And since $\E{\big[ \hat{v}_t^{p} - \E{[\Bar{g}_t^2]} \big]} = \E{[\Bar{g}_t^2]} + \big(\frac{\sigma C}{B} \big)^2 - \E{[\Bar{g}_t^2]} = \big(\frac{\sigma C}{B} \big)^2 \nrightarrow 0$, but $\E{\big[ \hat{v}_t^{\textnormal{corr}} - \E{[\Bar{g}_t^2]} \big]} = (\E{[\Bar{g}_t^2]} - \big(\frac{\sigma C}{B} \big)^2) + \big(\frac{\sigma C}{B} \big)^2 - \E{[\Bar{g}_t^2]} = 0$, we have $\hat{v}_t^{p} \nrightarrow \E{[\Bar{g}_t^2]}$ but $\hat{v}_t^{\textnormal{corr}} \xrightarrow{p} \E{[\Bar{g}_t^2]}$.

Let $g(x)=1/\sqrt{x}, x\geq 0$. By the continuous mapping theorem, $\frac{1}{\sqrt{\hat{v}_t^{\textnormal{corr}}}} \xrightarrow{p} \frac{1}{\sqrt{\E{[\Bar{g}_t^2]}}}$. Then since $\hat{m}_t^{p} \xrightarrow{p} \E{[\Bar{g}_t]}$ and $\frac{1}{\sqrt{\hat{v}_t^{\textnormal{corr}}}} \xrightarrow{p} \frac{1}{\sqrt{\E{[\Bar{g}_t^2]}}}$, by joint convergence in probability, $(\hat{m}_t^{p}, \frac{1}{\sqrt{\hat{v}_t^{\textnormal{corr}}}}) \xrightarrow{p} (\E{[\Bar{g}_t]}, \frac{1}{\sqrt{\E{[\Bar{g}_t^2]}}})$.
Let $g(x,y)=xy$. Applying the continuous mapping theorem again yields $\frac{\hat{m}_t^p}{\sqrt{\hat{v}_t^{\textnormal{corr}}}} \xrightarrow{p} \frac{\E{[\Bar{g}_t]}}{\sqrt{\E{[\Bar{g}_t^2]}}}$.
\end{proof}

\section{Concentration of $\hat{v}_t^p$}
\label{apdix:bound}

First we introduce the following lemma which is used in both proofs of Proposition \ref{prop:bound_fixed} and \ref{prop:bound_martingale}.
\begin{lemma}
\label{lemma:zk2_sub_exponential}
    Let $Z \sim \cN(0, (1/B^2)\sigma^2 C^2)$, $\beta \in [0,1]$ be a constant and $B, C$ be constants such that $B>0$ and $C>0$, then $\beta Z^2$ is sub-exponential with $\nu=2\beta(1/B^2)\sigma^2C^2, b=4\beta(1/B^2)\sigma^2C^2$.
\end{lemma}
\begin{proof}
    Since $Z \sim \cN(0, (1/B^2)\sigma^2 C^2)$, let $X = \frac{Z}{(1/B)\sigma C} \sim \cN(0,1)$ and $a=\beta(1/B^2)\sigma^2C^2$ be a temporary constant,
    \begin{align*}
        &\E{\Big[ \exp{\bigl\{ \lambda (\beta Z^2 - \E{[\beta Z^2]}) \bigl\} }\Big]} \\
        &= \E{\Big[ \exp{\bigl\{ \lambda (a X^2  - a) \bigl\}}\Big]} \\
        &= \frac{1}{\sqrt{2\pi}} \int_{-\infty}^{\infty} \exp{ \bigl\{ \lambda (a z^2 - a) \bigl\}} \exp{\bigl\{-z^2/2\bigl\} } dz \\
        &= \frac{e^{-\lambda a}}{\sqrt{2\pi}} \int_{-\infty}^{\infty} \exp{\bigl\{ -\frac{z^2}{2}(1-2\lambda a)\bigl\} } dz \\
        &= \frac{e^{-\lambda a}}{\sqrt{1-2\lambda a}} \frac{1}{\sqrt{2\pi}} \int_{-\infty}^{\infty} e^{-y^2/2} dy, \; y=\sqrt{1-2\lambda a}z \\
        &= \frac{e^{-\lambda a}}{\sqrt{1-2\lambda a}}, \; \lambda < \frac{1}{2a} \\
        & \leq e^{\lambda^2 \nu^2/2}, \; \textnormal{for } \nu \geq \sqrt{2}a.
    \end{align*}
    The constant $\sqrt{2}$ comes from taking Taylor expansion around 0 of $\frac{e^{-\lambda a}}{\sqrt{1-2\lambda a}}$ and $e^{\lambda^2 \nu^2/2}$,
    \begin{equation*}
        \frac{e^{-\lambda a}}{\sqrt{1-2\lambda a}} = 1 + a^2\lambda^2 + o(\lambda^2), \; e^{\lambda^2 \nu^2/2} = 1 + \frac{1}{2}\lambda^2 \nu^2 + o(\lambda^2).
    \end{equation*}
    For any $v \geq \sqrt{2}a$ the last inequality would hold as $\frac{1}{2}\lambda^2 \nu^2 \geq a^2\lambda^2$, we pick $\nu = 2a$ and $b=4a$.
\end{proof}

\subsection{Proof of Proposition \ref{prop:bound_fixed}}\label{appendix:bound_fixed}
\begin{proposition*}[\ref{prop:bound_fixed}]
Consider a fixed-in-advance sequence of model parameters $\theta_t$ and mini-batches.
For $0 < \alpha < 1$, for each dimension $i$, we have $\sP[\lvert
 \hat{v}_t^{\textnormal{corr}} - \hat{v}_t^c \rvert_i \geq \xi ] \leq \alpha$ with:
\begin{align*}
    \xi \geq
    \begin{cases}
        (\frac{1-\beta_2}{1-\beta_2^t}) \sqrt{\ln{(1/\frac{\alpha}{2})}(2v^2)}
        & 0 \leq \frac{\xi(1-\beta_2^t)}{1-\beta_2}  \leq \frac{\nu^2}{b} \\
        (\frac{1-\beta_2}{1-\beta_2^t}) \ln{(1/\frac{\alpha}{2})}2b
        & \frac{\xi(1-\beta_2^t)}{1-\beta_2}    \geq \frac{\nu^2}{b},
    \end{cases}
\end{align*}
where $\nu = (\frac{4\sigma^2C^2}{B^2})\sqrt{\frac{1-\beta_2^{2t}}{1-\beta_2^2}}, b=\frac{4\sigma^2C^2}{B^2}$.
\end{proposition*}

\begin{proof}
Let $Z_{\tau} \sim \cN(0, (1/B^2)\sigma^2 C^2)$ be the independent DP noise drawn from a Normal distribution,
$Z_{\tau} \sim \cN(0, (1/B^2)\sigma^2 C^2) \implies X_{\tau} = \frac{Z_{\tau}}{(1/B)\sigma C} \sim \cN(0,1)$.
Let $Z_t$ be the random variable of the step $t$ moving average of independent squared noise, $Z_t \coloneqq (\frac{1-\beta_2}{1-\beta_2^t})\sum_{\tau=1}^{t} \beta_2^{t-\tau} Z_\tau^2$.

By Lemma \ref{lemma:zk2_sub_exponential}, let $a=\beta_{2}^{t-\tau}(1/B^2)\sigma^2C^2$ be a temporary constant, $\beta_2^{t-\tau}Z_\tau^2$ is sub-exponential with $\nu_{\tau}=2a, b_{\tau}=4a$.
By the additive property of sub-exponential and since $\{Z_\tau^2, \tau={1, ..., t}\}$ are independent, $\sum_{\tau=1}^{t}\beta_2^{t-\tau}Z_{\tau}^2$ is sub-exponential with
\begin{gather*}
    \nu^{*} =\sqrt{\sum_{\tau=1}^{t} \nu_{\tau}^2} = \sqrt{\sum_{\tau=1}^{t} 4a^2} = \bigg(\frac{2\sigma^2C^2}{B^2}\bigg)\sqrt{\frac{1-\beta_2^{2t}}{1-\beta_2^2}}, \\
    b^{*} = \max_{\tau=1,...,t}b_\tau = \frac{4\sigma^2C^2}{B^2}.
\end{gather*}
By Proposition 2.9 in \citet{wainwright2019high},
\begin{align*}
    \sP&[\lvert Z_t - \E{[Z_t]} \rvert  \geq (\frac{1-\beta_2}{1-\beta_2^t}) \delta] \\
    &\leq
    \begin{cases}
      2\exp{ \bigl\{ -\delta^2/(2\nu^{*2}) \bigl\} } & 0 \leq \delta \leq \frac{\nu^{*2}}{b^{*}} \\
      2\exp{ \bigl\{ -\delta/(2b^{*}) \bigl\} } & \delta > \frac{\nu^{*2}}{b^{*}}.
   \end{cases}
\end{align*}
For any tolerance level $\alpha < 1$, and with $\xi = (\frac{1-\beta_2}{1-\beta_2^t}) \delta$, we have $\sP[\lvert
 \hat{v}_t^{\textnormal{corr}} - \hat{v}_t^c \rvert_i \geq \xi ] \leq \alpha$ with:
\begin{align*}
    \xi \geq
    \begin{cases}
        (\frac{1-\beta_2}{1-\beta_2^t}) \sqrt{(-\ln{\frac{\alpha}{2}})(2\nu^{*2})}
        & 0 \leq \frac{\xi(1-\beta_2^t)}{1-\beta_2}  \leq \frac{\nu^{*2}}{b^{*}} \\
        (\frac{1-\beta_2}{1-\beta_2^t}) (-\ln{\frac{\alpha}{2}})2b^{*}
        & \frac{\xi(1-\beta_2^t)}{1-\beta_2}    \geq \frac{\nu^{*2}}{b^{*}}.
    \end{cases}
\end{align*}
\end{proof}

\subsection{Proof of Proposition \ref{prop:bound_martingale}}\label{appendix:bound_martingale}

We now provide high probability bounds for the error incurred when estimating $\hat{v}_t$.
In the DP optimization procedures we consider, the sequence of batches is drawn independently from everything else. We can thus understand the procedure as first drawing a sequence of batches, and then proceeding with DP optimization on this sequence. Without loss of generality, we fix the sequence of batches, and analyze the behavior of $\hat{v}_t$ on a fixed (unknown) sequence of batches $b_t$. The main reason is to ensure that the source of randomness in $\theta_t$ comes from the DP noise draw only, i.e. given step $t$ and a sequence of realized sample draw of DP noise, the parameters $\theta_t$ and thus $\overline{g}_t = (1/B)\sum_{n \in b_t} g_{n}(\theta_t) / {\max{(1, \lVert g_n(\theta_t) \rVert_{2}/C})}$ are deterministic. Then, we have that:

\begin{proposition*}[\ref{prop:bound_martingale}]
        For $0 < \alpha < 1$, for each dimension $i$, we have $\sP[\lvert \hat{v}_t^p - \E{[\hat{v}_t^p]} \rvert_i \geq \xi] \leq \alpha$ with:
    \begin{align*}
    \xi \geq
    \begin{cases}
        (\frac{1-\beta_2}{1-\beta_2^t}) \sqrt{\ln{(1/\frac{\alpha}{2})}(2v^2)}
        & 0 \leq \frac{\xi(1-\beta_2^t)}{1-\beta_2}  \leq \frac{\nu^2}{b} \\
        (\frac{1-\beta_2}{1-\beta_2^t}) \ln{(1/\frac{\alpha}{2})}2b
        & \frac{\xi(1-\beta_2^t)}{1-\beta_2}  \geq \frac{\nu^2}{b},
    \end{cases}
\end{align*}
where $\nu = 2\sqrt{\frac{1-\beta_2^{2t}}{1-\beta_2^2}}( \frac{\sigma^2C^2}{B^2} + \frac{\sigma C^2}{B} )$, $b = \frac{4\sigma^2C^2}{B^2}$.
\end{proposition*}

\begin{proof}
Remember that $\hat{v}_t^{p} = (\frac{1-\beta_2}{1-\beta_2^t})\sum_{\tau=1}^{t} \beta_{2}^{t-\tau} {\Tilde{g}_\tau}^{2} = (\frac{1-\beta_2}{1-\beta_2^t})\sum_{\tau=1}^{t} \beta_{2}^{t-\tau} \big(\bar{g}_\tau + z_\tau\big)^{2}$.
We use the Doob martingale construction to analyze $\hat{v}_t^{p}$'s deviation from its mean. The sequence $Z \triangleq \{Z_\tau\}_{\tau=1}^t$ is a sequence of independent random variables (the DP noise draws). Define $f(Z) = \sum_{\tau=1}^{t} \beta_{2}^{t-\tau} \big(\bar{g}_\tau + Z_\tau\big)^{2}$,
$Y_0 = \E[f(Z)]$, $Y_t = f(Z)$, and
$Y_k = \E[f(Z) | Z_1, \ldots, Z_k]$. $D_k \triangleq Y_k - Y_{k-1}$ is a martingale difference sequence w.r.t.
$Z_k \triangleq \{Z_\tau\}_{\tau=1}^k$. Let $z_1, \ldots, z_{k}$ be a sequence of realized samples of $Z_{k}$.
Given the definition, we have that:
\begin{align*}
D_k &= \E[f(Z) | Z_1, \ldots, Z_k] - \E[f(Z) | Z_1, \ldots, Z_{k-1}] \\
    &= \bigg(\sum_{\tau=1}^{k} \beta_2^{t-\tau}(\Bar{g}_\tau + z_\tau)^2 \\ & \ \ \ + \E\bigg[\sum_{\tau=k+1}^{t} \beta_2^{t-\tau}(\Bar{g}_\tau + Z_\tau)^2 \Big\vert Z_1, \ldots, Z_k \bigg] \bigg) \\
    & \ \ \ - \bigg(\sum_{\tau=1}^{k-1} \beta_2^{t-\tau}(\Bar{g}_\tau + z_\tau)^2 \\ & \ \ \ + \E\bigg[\sum_{\tau=k}^{t} \beta_2^{t-\tau}(\Bar{g}_\tau + Z_\tau)^2 \Big\vert Z_1, \ldots, Z_{k-1} \bigg] \bigg) \\
    &= \beta_{2}^{t-k} \big(\bar{g}_k + z_k\big)^{2} - \beta_2^{t-k}\E\Big[\big(\bar{g}_k + Z_k\big)^{2} \big\vert Z_1, \ldots, Z_{k-1}\Big] \\
    &= \beta_{2}^{t-k}\big(z_k^2 + 2\bar{g}_k z_k - \V(Z_k)\big),
\end{align*}
where $\E[\bar{g}_k^2 | Z_1, \ldots, Z_{k-1}] = \bar{g}_k^2$ since $\bar{g}_k$ is deterministic given previous DP noise draws and a fixed batch $b_k$.
Next we show that $D_k$ is sub-exponential, i.e. $(*) = \E{\big[ e^{\lambda (D_k - \E{[D_k]})} | Z_1, \ldots, Z_{k-1}\big]} \leq e^{\lambda^2 v_k^2/2}$ for some $v_k$.
By Lemma \ref{lemma:zk2_sub_exponential}, let $a=\beta_{2}^{t-k}(1/B^2)\sigma^2C^2$ be a temporary constant, $\beta_2^{t-\tau}Z_\tau^2$ is sub-exponential with $\nu=2a, b=4a$.
For the second part, let $X_k = 2 \beta_{2}^{t-k}\bar{g}_k Z_k \sim (0, (2 \beta_{2}^{t-k}\bar{g}_k)^2 (1/B^2) \sigma^2 C^2)$, therefore $X_k$ is sub-exponential with $\nu = 2 \beta_{2}^{t-k}\bar{g}_k (1/B) \sigma C \leq 2 \beta_{2}^{t-k} (1/B) \sigma C^2$ (since $\Bar{g}_k$ is bounded by $C$ by clipping) and $b=0$, i.e.
\begin{gather*}
    \E{\Big[ \exp{\bigl\{ 2\lambda \beta_{2}^{t-k}\bar{g}_k Z_k \bigl\} }\Big]} \leq e^{\lambda^2 \nu^2 / 2}, \\
    \nu = 2 \beta_{2}^{t-k}(1/B)\sigma C^2, \; \forall \, |\lambda| \leq \infty.
\end{gather*}
By the additive property of sub-exponential we get $D_k$ is sub-exponential with $\nu_k = \nu_{1k} + \nu_{2k}$ where $\nu_{1k} = 2\beta_{2}^{t-k}(1/B^2)\sigma^2C^2$, $\nu_{2k} = 2 \beta_{2}^{t-k} (1/B)\sigma C^2$, $b = 4a = 4\beta_{2}^{t-k}(1/B^2)\sigma^2C^2$, $|\lambda| < 1/b$.
By Theorem 2.3 in \cite{wainwright2019high}, Chapter 2 , we have that $\sum_{k=1}^{t} D_k$ is sub-exponential with
\begin{gather*}
    \nu^{*} = \sqrt{\sum_{k=1}^{t} (\nu_{1k} + \nu_{2k})^2} = 2\sqrt{\frac{1-\beta_2^{2t}}{1-\beta_2^2}}\bigg( \frac{\sigma^2C^2}{B^2} + \frac{\sigma C^2}{B} \bigg), \\
    b^{*} = \max_{k=1, \ldots, t} b_k = \frac{4\sigma^2C^2}{B^2}.
\end{gather*}
For all $\delta \geq 0$,
\begin{align*}
    \sP[\lvert \sum_{k=1}^{t} D_k &\rvert \geq \delta] = \sP[\lvert f(Z) - \E{[f(Z)]} \rvert \geq \delta] \\ &\leq
    \begin{cases}
      2\exp{ \bigl\{ -\delta^2/(2\nu^{*2}) \bigl\} } & 0 \leq \delta \leq \frac{\nu^{*2}}{b^{*}} \\
      2\exp{ \bigl\{ -\delta/(2b^{*}) \bigl\} } & \delta > \frac{\nu^{*2}}{b^{*}},
   \end{cases}
\end{align*}
i.e. For $\hat{v}_t^p = (\frac{1-\beta_2}{1-\beta_2^t})f(Z)$ and $\E{[\hat{v}_t^p]} = (\frac{1-\beta_2}{1-\beta_2^t})\E{[f(Z)]}$, $\forall$ tolerance level $\alpha \leq 1$, and with $\xi = (\frac{1-\beta_2}{1-\beta_2^t}) \delta$, we have $\sP[\lvert \hat{v}_t^p - \E{[\hat{v}_t^p]} \rvert_i \geq \xi] \leq \alpha$ with:
\begin{align*}
    \xi \geq
    \begin{cases}
        (\frac{1-\beta_2}{1-\beta_2^t}) \sqrt{(-\ln{\frac{\alpha}{2}})(2\nu^{*2})}
        & 0 \leq \frac{\xi(1-\beta_2^t)}{1-\beta_2}  \leq \frac{\nu^{*2}}{b^{*}} \\
        (\frac{1-\beta_2}{1-\beta_2^t}) (-\ln{\frac{\alpha}{2}})2b^{*}
        & \frac{\xi(1-\beta_2^t)}{1-\beta_2}  \geq \frac{\nu^{*2}}{b^{*}},
    \end{cases}
\end{align*}
\end{proof}

\subsection{Numerical Analysis}\label{appendix:num-analysis}
We compute the numerical values of the bound in Proposition \ref{prop:bound_fixed} and \ref{prop:bound_martingale} under different tolerance levels $\alpha$. We also empirically measure the deviance of the observed DP bias to its expected value $\Phi$ by measuring the absolute difference between $\hat{v}_t^p$ and $\hat{v}_t^c + \Phi$. Table \ref{tab:num_bound_fixed_grad}, \ref{tab:num_bound} and \ref{tab:empirical_bound} summarizes the corresponding values at different step $t$ on the SNLI dataset. We observe that the empirical values in Table \ref{tab:empirical_bound} are smaller comparing to $\Phi$, suggesting that the observed DP bias are quite concentrated around its mean, and subtracting $\Phi$ from $\hat{v}_t^p$ should be relatively accurate. We also observe that the value of the bound in Proposition \ref{prop:bound_fixed} are much closer to the empirical values than in Proposition \ref{prop:bound_martingale}. It suggests that the bound derived under the fixed sequence of model parameters (Proposition \ref{prop:bound_fixed}) might be empirically more practical than the bound derived under the martingale assumptions (Proposition \ref{prop:bound_martingale}).

\begin{table*}[tb]
\centering
\renewcommand{\arraystretch}{1.1}
\resizebox{0.80\textwidth}{!}{%
\begin{tabular}{c|c|cccc}
\hline \hline
\textbf{Experiment} & \textbf{Quantity} & \bm{$t=10$} & \bm{$t=100$} & \bm{$t=1000$} & \bm{$t=10000$} \\ \hline
\multirow{3}{*}{\begin{tabular}[c]{@{}c@{}}$B=256, C=0.1$, \\ $\sigma=0.4, \beta_2=0.999$\end{tabular}}
 & $\sP[ \lvert \hat{v}_t^{\textnormal{corr}} - \hat{v}_t^c \rvert \geq ?] \leq 0.01$ & 1.005e-07 & 3.180e-08  & 1.046e-08 & 7.110e-09 \\
 & $\sP[ \lvert \hat{v}_t^{\textnormal{corr}} - \hat{v}_t^c \rvert \geq ?] \leq 0.05$ & 8.388e-08 & 2.654e-08 & 8.725e-09 & 5.933e-09  \\
 & $\sP[\lvert \hat{v}_t^{\textnormal{corr}} - \hat{v}_t^c \rvert \geq ?] \leq 0.10$ & 7.559e-08 & 2.391e-08 & 7.863e-09 & 5.347e-09 \\ \hline
\multirow{3}{*}{\begin{tabular}[c]{@{}c@{}}$B=256, C=1.0$, \\ $\sigma=0.4, \beta_2=0.999$\end{tabular}}
 & $\sP[\lvert \hat{v}_t^{\textnormal{corr}} - \hat{v}_t^c \rvert \geq ?] \leq 0.01$ & 1.005e-06 & 3.180e-06  & 1.046e-06 & 7.110e-07  \\
 & $\sP[\lvert \hat{v}_t^{\textnormal{corr}} - \hat{v}_t^c \rvert \geq ?] \leq 0.05$ & 8.388e-06 & 2.654e-06 & 8.725e-07 & 5.933e-07 \\
 & $\sP[\lvert \hat{v}_t^{\textnormal{corr}} - \hat{v}_t^c \rvert \geq ?] \leq 0.10$ & 7.559e-06 & 2.391e-06 & 7.863e-07 & 5.347e-07  \\ \hline \hline
\end{tabular}%
}
\caption{Numerical values for relevant quantities of the bound as in Proposition \ref{prop:bound_fixed}.}
\label{tab:num_bound_fixed_grad}
\end{table*}

\begin{table*}[tb]
\centering
\renewcommand{\arraystretch}{1.1}
\resizebox{0.80\textwidth}{!}{%
\begin{tabular}{c|c|cccc}
\hline \hline
\textbf{Experiment} & \textbf{Quantity} & \bm{$t=10$} & \bm{$t=100$} & \bm{$t=1000$} & \bm{$t=10000$} \\ \hline
\multirow{3}{*}{\begin{tabular}[c]{@{}c@{}}$B=256, C=0.1$, \\ $\sigma=0.4, \beta_2=0.999$\end{tabular}}
 & $\sP[\lvert \hat{v}_t^p - \E{[\hat{v}_t^p]} \rvert \geq ?] \leq 0.01$ & 3.222e-05 & 1.019e-05 & 3.351e-06 & 2.279e-06 \\
 & $\sP[\lvert \hat{v}_t^p - \E{[\hat{v}_t^p]} \rvert \geq ?] \leq 0.05$ & 2.688e-05 & 8.505e-06 & 2.797e-06 & 1.902e-06 \\
 & $\sP[\lvert \hat{v}_t^p - \E{[\hat{v}_t^p]} \rvert \geq ?] \leq 0.10$ & 2.423e-05 & 7.664e-06 & 2.520e-06 & 1.714e-06 \\ \hline
\multirow{3}{*}{\begin{tabular}[c]{@{}c@{}}$B=256, C=1.0$, \\ $\sigma=0.4, \beta_2=0.999$\end{tabular}}
 & $\sP[\lvert v_t^p - \E{[v_t^p]} \rvert \geq ?] \leq 0.01$ & 3.222e-03 & 1.019e-03 & 3.351e-04 & 2.279e-04 \\
 & $\sP[\lvert v_t^p - \E{[v_t^p]} \rvert \geq ?] \leq 0.05$ & 2.688e-03 & 8.505e-04 & 2.797e-04 & 1.902e-04 \\
 & $\sP[\lvert v_t^p - \E{[v_t^p]} \rvert \geq ?] \leq 0.10$ & 2.423e-03 & 7.664e-04 & 2.520e-04 & 1.714e-04 \\ \hline \hline
\end{tabular}%
}
\caption{Numerical values for relevant quantities of the bound as in Proposition \ref{prop:bound_martingale}.}
\label{tab:num_bound}
\end{table*}

\begin{table*}[tb]
\centering
\renewcommand{\arraystretch}{1.1}
\resizebox{0.80\textwidth}{!}{%
\begin{tabular}{c|c|cccc}
\hline \hline
\textbf{Experiment} & \textbf{Quantity} & \bm{$t=10$} & \bm{$t=100$} & \bm{$t=1000$} & \bm{$t=10000$} \\ \hline
\multirow{3}{*}{\begin{tabular}[c]{@{}c@{}}$B=256, C=0.1$, \\ $\sigma=0.4, \beta_2=0.999$ \end{tabular}}
 & $\sP[|v_t^p - \E{[v_t^p]}| \geq ?] = 0.01$ & 3.266e-08 & 9.004e-09 & 2.940e-09 & 2.052e-09 \\
 & $\sP[|v_t^p - \E{[v_t^p]}| \geq ?] = 0.05$ & 2.063e-08 & 6.181e-09 & 2.107e-09 & 1.492e-09 \\
 & $\sP[|v_t^p - \E{[v_t^p]}| \geq ?] = 0.10$ & 1.493e-08 & 4.746e-09 & 1.670e-09 & 1.197e-09 \\ \hline
\multirow{3}{*}{\begin{tabular}[c]{@{}c@{}}$B=256, C=1.0$, \\ $\sigma=0.4, \beta_2=0.999$ \end{tabular}}
 & $\sP[|v_t^p - \E{[v_t^p]}| \geq ?] = 0.01$ & 3.266e-06 & 9.004e-07 & 2.940e-07 & 2.052e-07 \\
 & $\sP[|v_t^p - \E{[v_t^p]}| \geq ?] = 0.05$ & 2.064e-06 & 6.181e-07 & 2.107e-07 & 1.492e-07 \\
 & $\sP[|v_t^p - \E{[v_t^p]}| \geq ?] = 0.10$ & 1.493e-06 & 4.746e-07 & 1.670e-07 & 1.197e-07 \\ \hline \hline
\end{tabular}%
}
\caption{Empirically measured values for deviance of the observed DP bias from $\Phi$.}
\label{tab:empirical_bound}
\end{table*}

\section{More Results}
\label{apdix:more_res}

Figure \ref{fig:exp1} shows the mean $\pm$ standard deviation of test accuracy over the privacy budget ($\epsilon$-DP) on three datasets, repeated over 5 runs with different random seeds.
On SNLI, DP-AdamBC performs better than DP-Adam: the accuracy improves from 52.63\% to 56.08\% ($3.5$ percentage points). Both perform much better than DP-SGD (51.04\% on SNLI).
On QNLI, we observe a similar behaviour which the accuracy improves from 61.23\% to 62.83\% ($1.6$ percentage points) with the bias correction, and both perform better than DP-SGD (58.29\%).
For CIFAR10, on which Adam often performs worse than SGD in non-private settings, DP-Adam and DP-AdamBC (62.24\% vs 63.43\% accuracy) performs similarly and are both worse than DP-SGD (65.30\% accuracy).
On ogbn-arxiv, DP-Adam and DP-AdamBC perform similarly (54.02\% vs 53.81\% accuracy) and are outperformed by DP-SGD (54.20\%).

\paragraph{DP-Adam has similar performance to DP-SGDM.}
We provide more details about the comparison between DP-Adam and DP-SGDM with a converted learning rate schedule as in Equation \ref{eq:lr_convert}. Figure \ref{fig:exp_rel_sgdm_more} shows the train loss, test loss and test accuracy between the two algorithms on SNLI (top row) and two experiment setups with different $\Phi$ on CIFAR10 (middle and bottom row). The solid line and the shaded area show the mean and standard deviation over 5 repeated runs. For SNLI (top row), it was run with $\eta=0.01$ for DP-Adam and $\eta \approx 6.4$ for DP-SGDM with $B=256, C=0.1, \sigma=0.4, \beta_1=0.9$. For CIFAR10 with relatively small $\Phi$ (middle row), it was run with $\eta=0.001$ for DP-Adam and $\eta\approx0.2048$ for DP-SGDM with $B=2048, C=1.0, \sigma=1.0, \beta_1=0.9$; for CIFAR10 with relatively large $\Phi$ (bottom row), it was run with $\eta=0.001$ for DP-Adam and $\eta \approx0.0256$ for DP-SGDM with $B=256, C=1.0, \sigma=1.0, \beta_1=0.9$. We observe that the performances are close between the two algorithms in train loss, test loss and test accuracy. Some discrepancy still exists since the observed value of DP bias is concentrated around but not exactly equal to $\Phi$. When $\Phi$ is relatively large as in the second case on CIFAR10, $\Phi$ is more likely to dominate the denominator of DP-Adam's update, and we observe an even closer behaviour between DP-Adam and DP-SGDM in all three aspects. There could also be possible different generalization behaviour between the two algorithm, but the training behaviour is almost identical.

\begin{figure*}[t]
\centering
\includegraphics[width=0.9\linewidth]{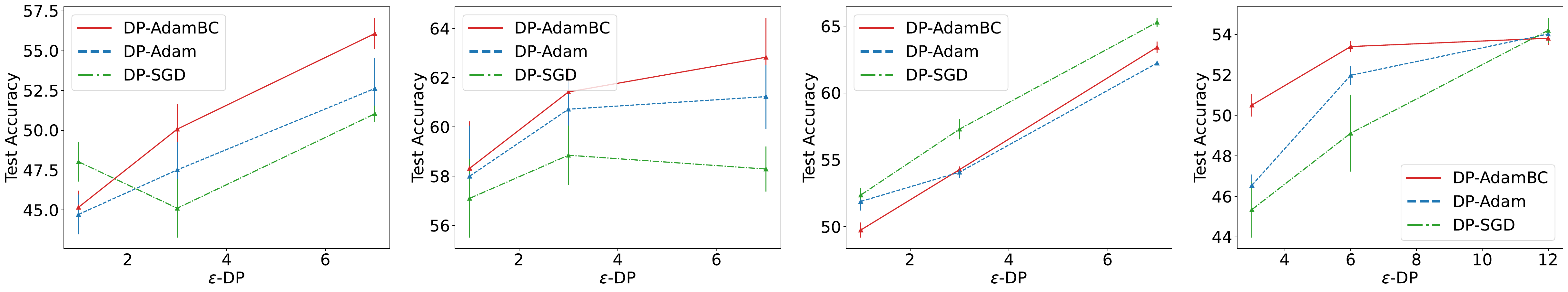}
\caption{(From left to right) Comparing the performance of DP-Adam, DP-AdamBC and DP-SGD on QNLI and SNLI dataset (nlp), CIFAR10 (images) and obgn-arxiv (node classification) at different target pribacy budget ($\epsilon$). Each result is tuned separately. We report the mean (standard deviation) over 5 runs for the best parameters.}
\label{fig:res_multi_eps}
\end{figure*}

\begin{figure*}[tb]
\centering
\includegraphics[width=0.7\linewidth]{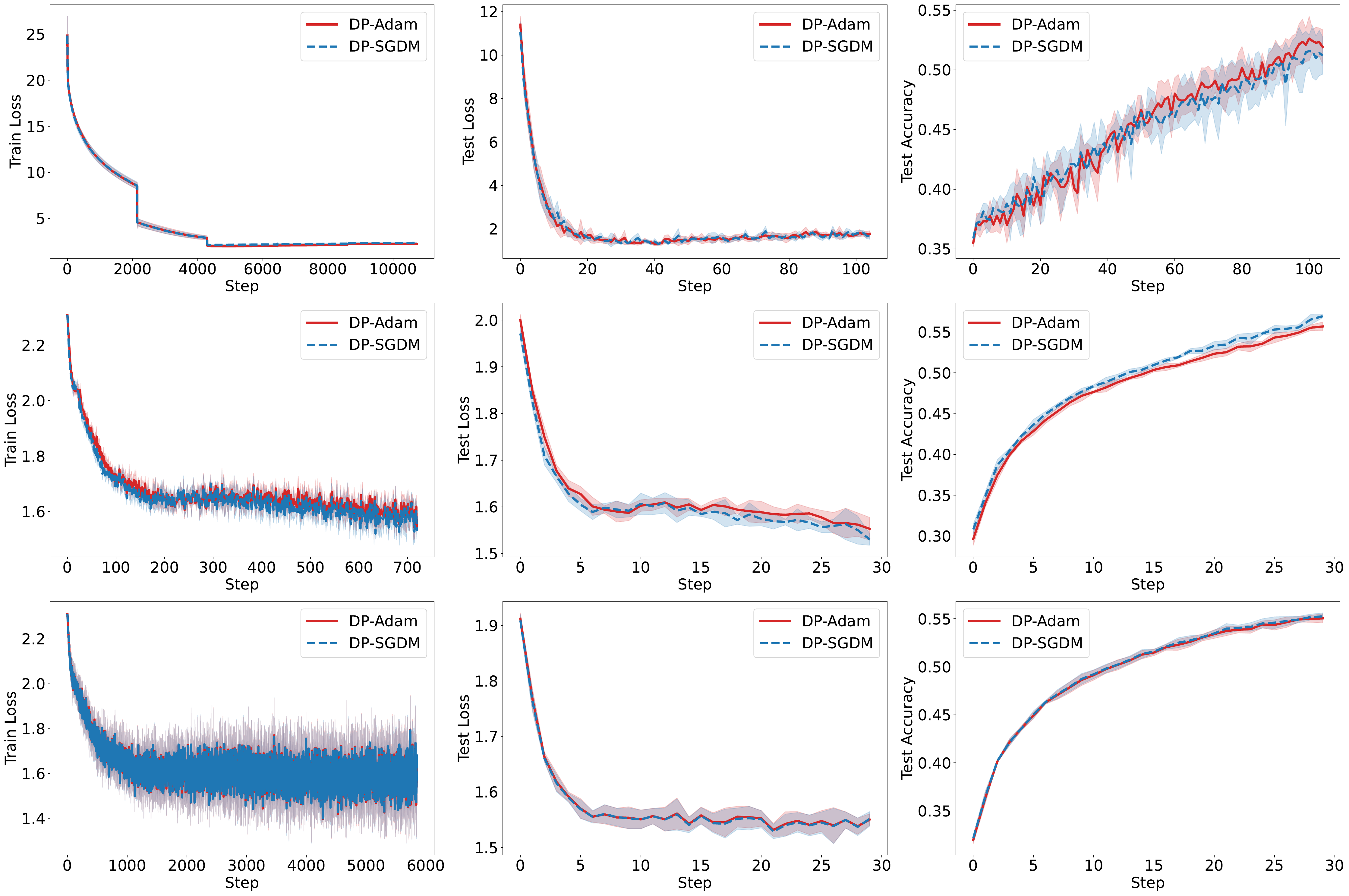}
\caption{Compare train loss, test loss and test accuracy between DP-Adam and the DP-SGDM with the converted learning rate schedule as in Equation \ref{eq:lr_convert} on SNLI (top row) and CIFAR10 with relatively small $\Phi$ (middle row) and large $\Phi$ (bottom row) respectively.
}
\label{fig:exp_rel_sgdm_more}
\end{figure*}

\paragraph{Additional Experiments.}
\label{apdix:additional_exp}
We repeat the comparison between DP-AdamBC, DP-Adam and DP-SGD on SNLI dataset with target $\epsilon=3$. The hyperparameters are $lr_{\textnormal{DP-Adam}} = 0.005$, $\gamma'_{\textnormal{DP-Adam}} = 1\textnormal{e-8}$, $lr_{\textnormal{DP-AdamBC}} = 0.005$ , $\gamma_{\textnormal{DP-AdamBC}} = 5\textnormal{e-9}$, $B=256, C=0.1, \sigma=0.5$. For each algorithm, we report the mean and standard deviation over 5 repeated runs with different seeds. Figure \ref{fig:additional_exp}(Left) shows that DP-AdamBC is approximately $2.6\%$ better in final mean test accuracy than DP-Adam with this privacy budget. The final mean(standard deviation) test accuracy for DP-AdamBC and DP-Adam are $50.1\%(1.6\%)$, $47.5\%(1.8\%)$ respectively.

We conduct additional experiments to compare the performances between DP-AdamBC, DP-Adam, and DP-SGD for smaller privacy budgets for all datasets. For CIFAR10, we tune the relevant hyperparameters (learning rate, $\gamma$, $\gamma'$, C, B, and number of steps) independently for each experiment. For SNLI and QNLI, we tune the same learning rate except for $B=256$ is fixed at its maximum capacity allowed on our machine.
For ogbn-arxiv, we tune the model hyperparameters for DP-Adam/DP-AdamBC and DP-SGD over 50 runs and then tune the optimizer hyperparameters for each algorithm (with fixed model hyperparameters). For DP-Adam/DP-AdamBC, we use noise multiplier $\lambda = 2.7$, maximum degree $K = 6$, batch size $B = 10,000$, clipping norm $C = 0.0003$, one encoder layer, and two decoding layers. For DP-SGD, we use noise multiplier $\lambda = 2.3$, maximum degree $K = 4$, batch size $B = 2000$, clipping norm $C = 0.0015$, one encoder layer, and one decoder layer. The optimizer hyperparameters are $lr_{\textnormal{DP-Adam}}~=~0.1$, $\gamma'_{\textnormal{DP-Adam}}~=~1\textnormal{e-8}$, $lr_{\textnormal{DP-AdamBC}}~=~0.0001$, $\gamma_{\textnormal{DP-AdamBC}}~=~4.3\textnormal{e-9}$, $lr_{\textnormal{DP-SGD}}~=~0.24$. We report the mean and standard deviation of 5 repeated runs with different seeds for the tuned algorithms. Table \ref{tab:res_multi_eps} shows that DP-AdamBC has the best mean test accuracy at this privacy budget.

To examine the generalizability of the results we test the algorithm on larger dataset and model. We repeat the comparison between DP-AdamBC and DP-Adam on SST2 dataset and Bert-Large model, with the last encoder block and the classifier head randomly initialized and trained. The hyperparameters are $lr_{\textnormal{DP-Adam}} = 0.005$, $\gamma'_{\textnormal{DP-Adam}} = 1\textnormal{e-8}$, $lr_{\textnormal{DP-AdamBC}} = 0.003$ , $\gamma_{\textnormal{DP-AdamBC}} = 3\textnormal{e-9}$, $B=256, C=0.1, \sigma=0.4$. Figure \ref{fig:additional_exp}(Right) shows that DP-AdamBC has around 2.4\% advantage in final mean test accuracy than DP-Adam.

\begin{figure*}[htb]
\centering
\includegraphics[width=\linewidth]{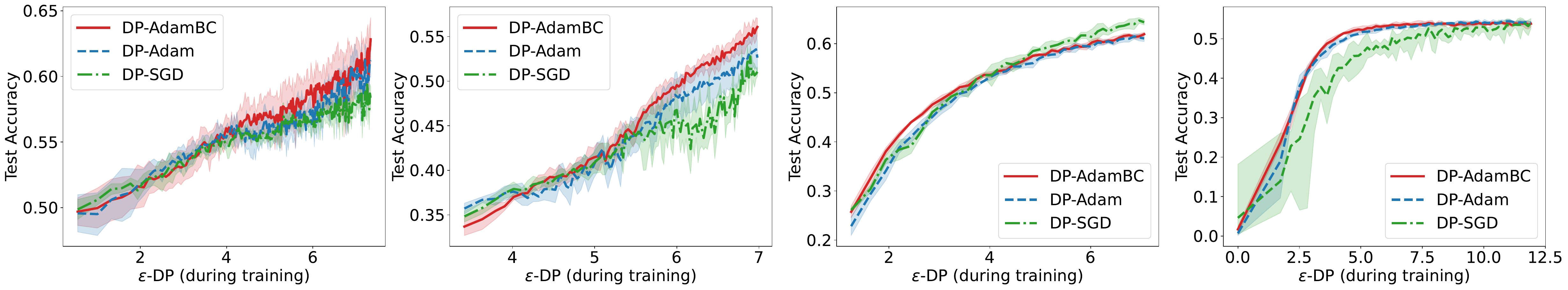}
\caption{(From left to right) Comparing the performance of DP-Adam, DP-AdamBC and DP-SGD on QNLI and SNLI dataset (nlp), CIFAR10 (images) and obgn-arxiv (node classification). At the end of training $\epsilon\textnormal{-DP} \approx 7$ for CIFAR10, QNLI and SNLI and $\epsilon\textnormal{-DP} \approx 12$ for obgn-arxiv. Each optimizer is tuned separately. The x-axis is the step over a single training trajectory converted to privacy budget $\epsilon$ to make results comparable for different optimizers.}
\label{fig:exp1}
\end{figure*}

\begin{figure*}[htb]
     \centering
     \begin{subfigure}[b]{0.35\textwidth}
         \centering
         \includegraphics[width=\textwidth]{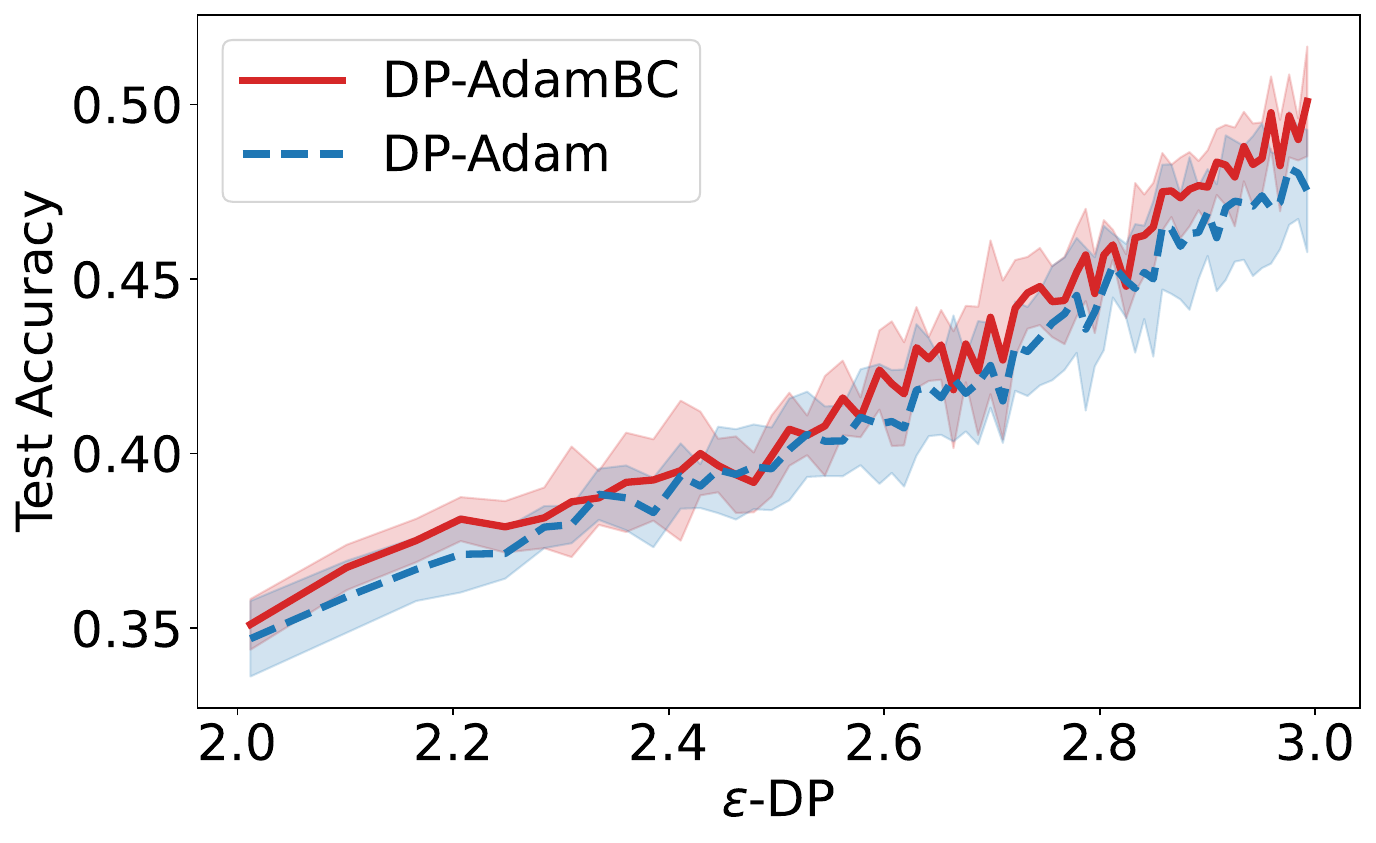}
     \end{subfigure}
     \begin{subfigure}[b]{0.35\textwidth}
         \centering
         \includegraphics[width=\textwidth]{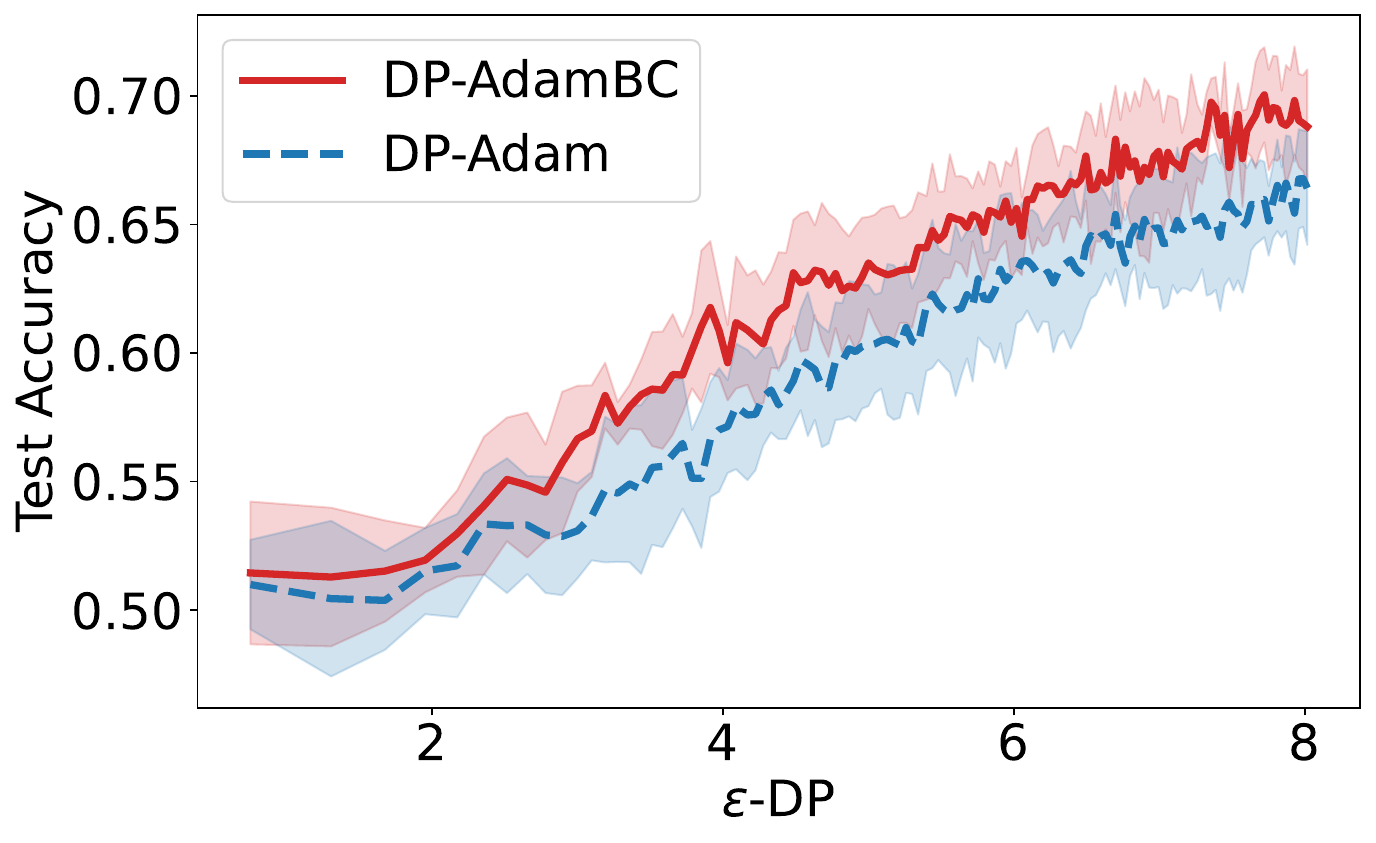}
     \end{subfigure}
    \caption{(\textbf{Left: }) Comparison between DP-AdamBC and DP-Adam when target $\epsilon=3$ on the SNLI dataset. DP-AdamBC shows $2.6\%$ advantage in final mean test accuracy. The final mean(standard deviation) test accurary for DP-AdamBC and DP-Adam are $50.1\%(1.6\%)$, $47.5\%(1.8\%)$ respectively. (\textbf{Right: }) Comparison between DP-AdamBC and DP-Adam on the SST2 dataset with Bert-Large. DP-AdamBC shows 2.4\% advantage in final test accuracy compared to DP-Adam (68.80\% vs 66.4\%).}
    \label{fig:additional_exp}
\end{figure*}

\paragraph{Comparison to DP-Adam variant.}
\label{apdix:compare_dp2}

We performed an empirical comparison of $\textrm{DP}^2$ with DP-AdamBC, and discuss the differences between the two algorithms below (we refer to Algorithm 1 in \citet{lidp2} as $\textrm{DP}^2$). There are three key differences between $\textrm{DP}^2$ and DP-AdamBC:
\begin{enumerate}
    \item Adam does not use pre-conditioned gradients in its update, since the moments are estimated from non-scaled gradients, whereas RMSprop (the base for $\textrm{DP}^2$ in \citet{lidp2}) uses scaled (thus pre-conditioned) gradients in the update. Figure 10 in \citet{lidp2} shows that the gains obeseved in $\textrm{DP}^2$ come in large part from reducing the amount of clipping and noising by using pre-conditioned gradients, but such advantage cannot directly transfer to Adam.
    \item The pre-conditioning term ($D_t$) is computed on much larger batches of gradients, since it is accumulated over multiple iterations. This means that $D_t$ is computed on a different distribution than the actual $t$-step gradients. For Adam, it would imply that the first moments (expectation of $t$-step gradient) and second moments (variance of $t$-step gradient) are estimated with different sampling distributions which could break Adam's sign descent behaviour.
    \item $\textrm{DP}^2$ does not use momentum on the gradients (as the first moment in Adam) potentially because it alternates between two optimizers. Momentum is typically important on tasks for which Adam works well \citep{kunstner2023heavytailed}.
\end{enumerate}

Empirically, \citet{lidp2} evaluates $\textrm{DP}^2$ on linear models and matrix factorization, and not on deep learning tasks. Figure \ref{fig:compare_dp2_snli_only} (Left) and \ref{fig:compare_dp2_cifar_only} (Left)  shows the comparison between DP2-RMSProp, DP-AdamBC and DP-SGD on SNLI with Bert-base and CIFAR10 with CNN respectively. DP2-RMSProp introduces several new hyperparameters (learning rates, clipping thresholds and delay parameters in both phases of SGD and RMSProp) whereas DP-AdamBC adds no additional hyperparameters compared to DP-Adam. We tuned $\textrm{DP}^2$ with grid search over: learning rate-\{0.001, 0.01, 0.1, 1.0, 3.0, 5.0, 7.0, 10.0\} for both datasets in both phases of $\textrm{DP}^2$, (CIFAR10) $s1=s2=s$ as suggested in \citet{lidp2}, s-\{25, 65, 130, 250\} (roughly 4, 10, 20, 40 out of 50 epochs), and C-\{0.1, 1, 3\} in both phases of $\textrm{DP}^2$; (SNLI) s-\{1000, 2000, 4500\} (roughly 0.5, 1, 2 out of 3 epochs) and C-\{0.01, 0.1, 1\} in both phases of $\textrm{DP}^2$, we fixed batch size $B$ and noise multiplier $\sigma$ to be the same as with the other three algorithms. Figure \ref{fig:compare_dp2_snli_only} (Left) and \ref{fig:compare_dp2_cifar_only} (Left) shows the mean and standard deviation of the test accuracy over 5 runs with on the two datasets. We observe that $\textrm{DP}^2$ first follows DP-SGD (since the first steps use this optimizer), and then struggles to converge on deep learning tasks, leading to poor performance. Indeed switching between two optimizers seem to make $\textrm{DP}^2$ unstable: Figure \ref{fig:compare_dp2_snli_only} (Right) and \ref{fig:compare_dp2_cifar_only} (Right) shows $\textrm{DP}^2$ on these two tasks with different $s$ (switching frequency). We observe the performance either has large turbulence or drops significantly when switching optimizers during training.



\begin{figure*}[htb]
    \centering
    \includegraphics[width=0.7\textwidth]{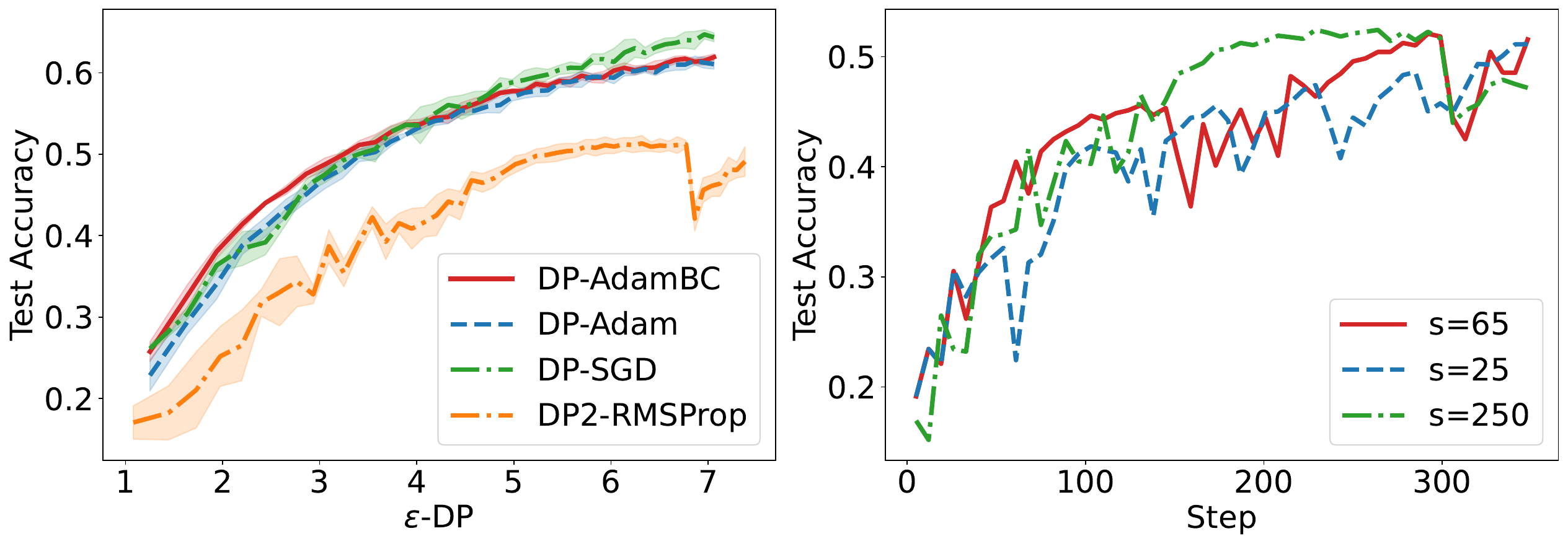}
    \caption{\textbf{Left:} Comparison between DP2RMSProp, DP-AdamBC, DP-Adam and DP-SGD on CIFAR10 with CNN, \textbf{Right:} The performance of DP2RMSProp with different phase switching frequency $s$ on CIFAR10 with CNN.}
    \label{fig:compare_dp2_cifar_only}
\end{figure*}

\section{Privacy Analysis}
\label{apdix:privacy_proof}
Since DP-AdamBC uses the privitized gradient to update first and second moment estimates, and the DP bias $\Phi$ can be calculated from public hyperparameters $B, \sigma, C$. By the post-processing property of DP, for a given privacy accounting method, the same privacy guarantee holds for DP-AdamBC as with DP-SGD or DP-Adam. We formalize the proof as follows.

\begin{theorem}[Privacy guarantee of DP-SGD]
There exist constants $c_1$ and $c_2$ so that given the sampling probability $q=L/N$ and the number of steps $T$, for any $\epsilon < c_1 q^2 T$ , Algorithm 1 in \citet{abadi2016deep} is $(\epsilon, \delta)$-differentially private for any $\delta>0$ if we choose $\sigma \geq (c_2 q \sqrt{T \log{(1/\delta)}}) / \epsilon$.
\end{theorem}

\begin{proposition}[Privacy guarantee of DP-AdamBC]
    Let the optimization algorithm DP-SGD($\theta, X, y, C, \sigma, B$) (Algorithm 1 in \citet{abadi2016deep}),
    with privacy analysis \textit{Compose}($T$, $\theta_{1,\ldots,T}$), be $(\epsilon, \delta)$-DP, then DP-AdamBC($\theta, X, y, C, \sigma, B$) with the same privacy analysis \textit{Compose}($T$, $\theta_{1,\ldots,T}$) is also $(\epsilon, \delta)$-DP.
\end{proposition}

\begin{proof}
    Let \textit{PrivitizeGradient}($\theta, X, y, C, \sigma, B$) be the key step providing DP guarantee in DP-SGD, DP-Adam and DP-AdamBC:
    \begin{center}
    (Compute gradient) $g_t(x_i) \xleftarrow{} \nabla_\theta \mathcal{L} (\theta, x_i), \forall i \in B$, \\
    (Clip gradient) $\Bar{g}_t(x_i) \xleftarrow{} g_t(x_i) / \max{(1, \lVert g_t(x_i) \rVert /C)}$, \\
    (Noise gradient) $\Tilde{g_t}\xleftarrow{} (1/B) (\sum_i \Bar{g}_t(x_i) + \mathcal{N}(0, \sigma^2 C^2))$.
    \end{center}
    When the DP noise is sampled from a Gaussian distribution, by standard arguments of the Gaussian mechanism in \citet{dwork2006calibrating} (Appendix A), the procedure is $(\epsilon', \delta')$-DP with $\sigma \geq (2\ln{(1.25/\delta')}C)/\epsilon'$. By the privacy amplification theorem, \textit{PrivitizeGradient}($\theta, X, y, C, \sigma, B$) is $(O(q\epsilon'), q\delta')$-DP with sampling probability $q=B/N$ for batch size $B$ and total sample size $N$. Let \textit{Compose}($T$, $\theta_{1,\ldots,T}$) computes the overall privacy cost over $T$ training iterations with a privacy accountant (e.g. strong composition in \citet{dwork2006calibrating}, moment accountant in \citet{abadi2016deep}, RDP accountant in \citet{Mironov_2019}), \textit{PrivitizeGradient}($\theta, X, y, C, \sigma$) over $T$ iterations with \textit{DP-SGD} update $\theta_{t+1} \xleftarrow{} \theta_{t} - \eta \nabla_{\theta_t} f$, where $\eta$ is a hyperparameter representing learning rate, $\nabla_{\theta_t} f$ is the output privitized gradient from \textit{PrivitizeGradient}, is $(\epsilon, \delta)$-DP. Since DP-AdamBC's update (Equation (4)) does not inquire additional private information as (1) $\Phi$ and $\gamma'$ are determined from user-defined hyperparameters and (2) moment estimates are calculated from privitized gradients. By the post-processing property of DP \citep{dwork2006calibrating}, \textit{PrivitizeGradient}($\theta, X, y, C, \sigma$) over $T$ iterations with \textit{DP-AdamBC} update is $(\epsilon, \delta)$-DP. If \textit{Compose}($T$, $\theta_{1,\ldots,T}$) is the moment accountant in \citet{abadi2016deep}, DP-AdamBC has the same privacy guarantees as in Theorem 1.
\end{proof}

\section{Convergence analysis}
\label{apdix:convergence_analysis}
In a recent work, \citet{défossez2022simple} shows dropping the correction term in the first moment estimation (Section 2.2) has no observable effect and proves the convergence of Adam with this mild modification. We use the same setup and extend the analysis from \citet{défossez2022simple}.

Let $F: \sR^d \rightarrow \sR$ be the objective function which $d \in \sN$ is the problem dimension (number of parameters of the function to optimize), $\theta$ be the model parameters, $f: \sR^d \rightarrow \sR$ be the stochastic function such that $\forall \theta \in \sR^{d}, \E{[\nabla f(\theta)]} = \nabla F(\theta)$, $\eta$ be the learning rate, the subscript $t$ be the step index, the subscript $i$ be the dimension index. We make the following assumptions.

\begin{assumption}[$F$ is bounded below]
\label{assmption:f_bounded_below}
$F$ is bounded below by $F_*$, i.e. $\forall \theta \in \sR^{d}, F(\theta) \geq F_*$.
\end{assumption}

\begin{assumption}[Bounded stochastic gradient]
\label{assmption:bounded_stochastic_grad}
The $L2$-norm of the stochastic gradient is uniformly almost surely bounded, i.e. $\exists C \geq 0$ such that $\forall t$, $\forall \theta \in \sR^{d}, \norm{\nabla f_t(\theta)} \leq C$ a.s..
\end{assumption}
Note that Assumption \ref{assmption:bounded_stochastic_grad} implies that for every update step the gradient clipping operation in DP-SGD, DP-Adam or DP-AdamBC has no effect. It is usually assumed for the ease of theoretical analysis such as in \citet{li2023dp2}, but is often violated empirically.

At step $t$, we call the noise sample $z_t \sim Z$. The privatized gradient is $\Tilde{g_t} = \nabla f_t(\theta) + z_t$, and the private second moment estimate is $v_t^p=\sum_{j=1}^{t}\beta_2^{j}\Tilde{g_t}^2$. We introduce the following Lemma, which is the main technical change we need to adapt the proof of \citet{défossez2022simple} to our setting with DP noise.

\begin{lemma}
\label{lemma:gaussian_concentration_on_z}
Let $\mu^{*} = (\beta_2(\beta_2^t-1)/(\beta_2-1))((\Phi-\frac{2\Phi}{\pi})+(C+\sqrt{\frac{2\Phi}{\pi}})^2)$, $\nu^{*} = (2\beta_2^2\Phi\sqrt{(\beta_2^t-1)/ (\beta_2^2-1)})$, $b^{*}=4\beta_2\Phi$ we have,
    $\Pr[v_t^p \geq \delta] \leq \alpha$ with,
    \begin{align*}
        \delta \geq
        \begin{cases}
            \mu^{*} + \sqrt{2\nu^{*2}\ln{(\frac{2}{\alpha})}}
            & 0 \leq \delta  \leq \frac{\nu^{*2}}{b^{*}} \\
            \mu^{*} + 2b^{*}\ln{(\frac{2}{\alpha})}
            & \delta   \geq \frac{\nu^{*2}}{b^{*}}.
        \end{cases}
    \end{align*}
\end{lemma}
\begin{proof}
    By Assumption \ref{assmption:bounded_stochastic_grad} we have $|\nabla f_t(\theta)| \leq C, \, \forall t$, so that $\beta_2^t(\nabla f_t(\theta) + z_t)^2 \leq \beta_2^t(|\nabla f_t(\theta)| + |z_t|)^2 \leq \beta_2^t(C + |z_t|)^2$.
    Since $z_t$ is sampled from $\cN(0, \Phi=(\sigma C/B)^2)$, by the similar argument as in Proposition \ref{prop:bound_fixed}, $\beta_2^t(C + z_t)^2$ is sub-exponential with $\nu = 2\beta_2^t\Phi$, $b=4\beta_2^t\Phi$. Since each $z_t$ is drawn independently, by the additive property of sub-exponential, $\sum_{j=1}^{t}\beta_2^{j}(C + z_t)^2$ is sub-exponential with $\nu^{*} = (2\beta_2^2\Phi\sqrt{(\beta_2^t-1)/ (\beta_2^2-1)})$, $b^{*}=4\beta_2\Phi$.
    In addition we have $\E{[v_t^p]} \leq \E{[\sum_{j=1}^{t}\beta_2^j(C+|Z_t|)^2]}=\sum_{j=1}^{t}\beta_2^j\E{[(C+|Z_t|)^2]} = \sum_{j=1}^{t}\beta_2^j (\Var[|Z_j|]+(C+\E{[|Z_j|]})^2)$. Given $Z_j\sim N(0,\Phi)$, $|Z_j|$ has a truncated Normal distribution where $\E{[|Z_j|]} = \sqrt{\frac{2\Phi}{\pi}}$ and $\Var{[|Z_j|]}=\Phi-\frac{2\Phi}{\pi}$. Let $\mu^{*} = (\beta_2(\beta_2^t-1)/(\beta_2-1))((\Phi-\frac{2\Phi}{\pi})+(C+\sqrt{\frac{2\Phi}{\pi}})^2)$ we have $E{[v_t^p]}\leq \mu^{*}$.
    By Proposition 2.9 in \citet{wainwright2019high} we have the result.
\end{proof}

We can apply a union bound over Lemma \ref{lemma:gaussian_concentration_on_z} to directly obtain the following result:
\begin{corollary}
\label{lemma:union_bound_on_v}
    Let $\mu^{*}$ be defined as in Lemma \ref{lemma:gaussian_concentration_on_z}. We have that $\Pr[\sup_t v_t^p \geq \delta] \leq \alpha$ with,
    \begin{align*}
        \delta \geq
        \begin{cases}
            \mu^{*} + \sqrt{2\nu^{*2}\ln{(\frac{2T}{\alpha})}}
            & 0 \leq \delta  \leq \frac{\nu^{*2}}{b^{*}} \\
            \mu^{*} + 2b^{*}\ln{(\frac{2T}{\alpha})}
            & \delta   \geq \frac{\nu^{*2}}{b^{*}}.
        \end{cases}
    \end{align*}
\end{corollary}


\begin{assumption}[Smoothness of $F$]
\label{assmption:f_smooth}
The gradient of $F$ is $L$-Liptchitz-continuous such that $\forall \theta, \theta' \in \sR^{d}, \norm{\nabla F(\theta) - \nabla F(\theta')} \leq L \norm{\theta - \theta'}$.
\end{assumption}

Let $\theta_0$ be the randomly initialized starting point, $\forall t \in [1, \ldots, T]$, the update rule of DP-AdamBC is $\theta_{t} \leftarrow \theta_{t-1} - \eta \cdot \hat{m}_t / \sqrt{\max{\big(\hat{v}_t - \Phi}, \gamma' \big)}$, and the update rule of DP-Adam is $\theta_{t} \leftarrow \theta_{t-1} - \eta \cdot \hat{m}_t / \sqrt{\hat{v}_t + \gamma}$, where $\hat{m}_t$, $\hat{v}_t$ and $\Phi$ are the private first, second moment estimates and the DP bias in private second moment estimates (Section \ref{sec:motiv-bias}). We prove the following propositions.

\begin{proposition}[Convergence of DP-AdamBC and DP-Adam Without momentum]
    Given the above assumptions and the defined update rule with $\beta_1 = 0, 0 < \beta_2 < 1$, $\eta_t = \eta \sqrt{\frac{1-\beta_2^t}{1-\beta_2}}, \eta>0$, $\forall \alpha$ s.t. $0<\alpha<1$, let $\nu^{*} = (2\beta_2^2\Phi\sqrt{(\beta_2^t-1)/ (\beta_2^2-1)})$, $b^{*}=4\beta_2\Phi$, we have $\forall T \in \sN^*$, \\
    \\
    (DP-Adam)
    \begin{align*}
        \E{[\norm{\nabla F(\theta_{t})}^2]} &\leq \frac{2\delta(F(\theta_0)-F_{*})}{\eta T} \\ &+ \bigg( \frac{4d(C^2+\Phi)}{\sqrt{1-\beta_2}} + \frac{\eta d L \sqrt{C^2+\Phi}}{1-\beta_2} \bigg) \\ & \times \bigg(\frac{1}{T}\ln{\big(1+\frac{C^2+\Phi}{(1-\beta_2)\epsilon}) - \ln{(\beta_2)}} \bigg),
    \end{align*}
    (DP-AdamBC)
    \begin{align*}
        \E{[\norm{\nabla F(\theta_{t})}^2]} &\leq \frac{2\sqrt{\delta^2 - \Phi}(F(\theta_0)-F_{*})}{\eta T} \\ &+ \bigg( \frac{4dC^2)}{\sqrt{1-\beta_2}} + \frac{\eta d L C}{1-\beta_2} \bigg) \\ & \times \bigg(\frac{1}{T}\ln{\big(1-\frac{C^2+\Phi}{(1-\beta_2)\Phi}) - \ln{(\beta_2)}} \bigg),
    \end{align*}
        \begin{align*}
    \delta \geq
    \begin{cases}
        \mu^{*} + \sqrt{\ln{(1/\frac{\alpha}{2T})}(2\nu^{*2})}
        & 0 \leq \delta  \leq \frac{\nu^{*2}}{b^{*}} \\
        \mu^{*} + \ln{(1/\frac{\alpha}{2T})}2b^{*}
        & \delta   \geq \frac{\nu^{*2}}{b^{*}}.
    \end{cases}
\end{align*}
\end{proposition}

\paragraph{DP-Adam.}
\begin{proof}
    By dropping the correction term in $m_t$ we have $\forall t, \eta_t = \eta\sqrt{\frac{1-\beta_2^t}{1-\beta_2}}$. Since Assumption \ref{assmption:bounded_stochastic_grad} assumes away the effect of gradient clipping, and since the DP noises are sampled from a zero-mean Normal distribution, proving convergence for DP-Adam (without Momentum) is very similar to the original proof of Theorem 2 in \citep{défossez2022simple}.
    We show the key steps below. Using Assumption \ref{assmption:f_smooth} we have,
    \begin{gather*}
        F(\theta_t) \leq F(\theta_{t-1}) - \eta_t \nabla F(\theta_{t-1})^T u_t + (\eta_t^2 L/2) \lVert u_t \rVert^2, \; \\
        u_t = \frac{\nabla_i f_t^{p}(\theta_{t-1})}{\sqrt{v_t^{p} + \epsilon}},
    \end{gather*}
    where $u_t$ is the update of DP-Adam without momentum.
    Taking the complete expectation with respect to all past steps before $t$ and $t$ step noise distribution we get,
    \begin{align*}
        \E{[F(\theta_t)]} \leq F(\theta_{t-1}) &- \eta_t \E{\bigg[\frac{\nabla_i F(\theta_{t-1}) \nabla_i f_t^{p}(\theta_{t-1})}{\sqrt{v_{t, i}^{p} + \epsilon }}\bigg]} \\ &+ \frac{\eta_t^2 L}{2}\E{[\lVert u_t \rVert^2]}.
    \end{align*}
    Given $\epsilon \ll v_t^p$, $\forall i \in [d]$  we can bound the first expectation term on the right side using Lemma \ref{lemma:gaussian_concentration_on_z}
    to get a high probability bound.
    By the similar steps as in Lemma 5.1 \citep{défossez2022simple}, except that we have $\E{[(\nabla_i f_t^{p}(\theta_{t-1}))^2]} = \E{[(\nabla_i f_t^{c}(\theta_{t-1}) + z_t)^2]} \leq C^2  + \Phi$, substituting the result of Lemma \ref{lemma:gaussian_concentration_on_z} gives the inequality on $\E{\bigg[\frac{\nabla_i F(\theta_{t-1}) \nabla_i f_t^{p}(\theta_{t-1})}{\sqrt{v_{t, i}^{p} + \epsilon }}\bigg]}$.
    Substituting in the result and since $\eta \leq \eta_t$ we get,
    \begin{align*}
        \E{[F(\theta_t)]} \leq F(\theta_{t-1}) &- \frac{\eta}{2\delta} \norm{\nabla F(\theta_{t-1})}^2 \\ &+ \bigg(2\eta_t \sqrt{C^2+\Phi} + \frac{\eta_t^2 L}{2}\bigg) \E{[\lVert u_{t} \rVert^2]},
    \end{align*}
    where $\delta$ has the form as in Lemma \ref{lemma:gaussian_concentration_on_z}.
    Summing over all steps $t \in [T]$ and taking the complete expectation, and since $\eta_t \leq \eta / \sqrt{1-\beta_2}$ we have,
    \begin{align*}
        \E{[F(\theta_T)]} &\leq F(\theta_0) - \frac{\eta \sum_{t=0}^{T-1} \E{[\norm{\nabla F(\theta_{t-1})}^2]}}{2\delta} \\ &+ \bigg(\frac{2\eta \sqrt{C^2+\Phi}}{\sqrt{1-\beta_2}} + \frac{\eta^2 L}{2(1-\beta_2)}\bigg) \sum_{t=0}^{T-1} \E{[\lVert u_{t} \rVert^2]},
    \end{align*}
    where $\delta$ has the form as in Corollary \ref{lemma:union_bound_on_v}.
    Since $y=ln(x)$ is a concave function for $x\in \mathbb{R}^{+}$, by Jensen's inequality we have $\E{(\ln{(x)})} \leq \ln{(\E{(x)})}$. By the similar steps as in Lemma 5.2 \citep{défossez2022simple} we have,
    \begin{align*}
        \sum_{t=0}^{T-1} \E{[\lVert u_{t, i} \rVert^2]} &= \sum_{i=1}^{d} \E{\bigg[ \sum_{t=0}^{T-1} u_{t, i}^2 \bigg]} \\ &\leq \sum_{i=1}^{d} \E{[\ln{(1+v_{T}^p/\epsilon)} - T\ln{(\beta_2)}]} \\
        &\leq \sum_{i=1}^{d} \big(\ln{(1+\frac{(C^2+\Phi)(1-\beta_2^{T})}{(1-\beta_2)\epsilon}) - T\ln{(\beta_2)}}\big)\\
        &\leq d\big(\ln{(1+\frac{C^2+\Phi}{(1-\beta_2)\epsilon}) - T\ln{(\beta_2)}}\big).
    \end{align*}
    Substituting in the result and rearrange the terms gives the final result.

\end{proof}

\paragraph{DP-AdamBC.}

\begin{proof}
    Similar to the proof steps above, by Assumption \ref{assmption:f_smooth} we have,
    \begin{gather*}
        F(\theta_t) \leq F(\theta_{t-1}) - \eta_t \nabla F(\theta_{t-1})^T u_t + (\eta_t^2 L/2) \lVert u_t \rVert^2, \; \\
        u_t = \frac{\nabla_i f_t^{p}(\theta_{t-1})}{\sqrt{v_t^{p} - \Phi}},
    \end{gather*}
    where $u_t$ is the update of DP-AdamBC without momentum.
    Taking the complete expectation we get,
    \begin{align*}
        \E{[F(\theta_t)]} \leq F(\theta_{t-1}) &- \eta_t \E{\bigg[\frac{\nabla_i F(\theta_{t-1}) \nabla_i f_t^{p}(\theta_{t-1})}{\sqrt{v_{t, i}^{p} - \Phi }}\bigg]}\\ &+ \frac{\eta_t^2 L}{2}\E{[\lVert u_t \rVert^2]}.
    \end{align*}
    We derive a high probability bound for the first expectation term on the right side.
    Let $G=\nabla_i F(\theta_{t-1}), g=\nabla_i f_t^{p}(\theta_{t-1}), \Bar{g}=\nabla_i f_t^{c}(\theta_{t-1})$, let $\Tilde{v}_t=\beta_2 v_{t-1}+\E{[g^2]}$ denote $v_t$ with the last gradient replaced by its conditional expectation with respect to all past steps, let $z, \Tilde{v}, v$ abbreviate for $z_t, \Tilde{v}_{t, i}^{p}, v_{t, i}^{p}$, we rewrite the first expectation term as follows,
    \begin{align*}
        \E{\bigg[ \frac{Gg}{\sqrt{v - \Phi}} \bigg]} &= \E{\bigg[ \frac{G(\Bar{g}+z)}{\sqrt{v - \Phi}} \bigg]} \\ &= \E{\bigg[ \frac{G\Bar{g}}{\sqrt{v - \Phi}} \bigg]} + \E{\bigg[ \frac{Gz}{\sqrt{v - \Phi}} \bigg]} \\ &= \E{\bigg[ \frac{G\Bar{g}}{\sqrt{v - \Phi}} \bigg]},
    \end{align*}
    since $G$ and $z$ are independent and $z$ has expectation equals zero. In addition, we have that $E{[g^2]} = \E{[(\Bar{g}+z)^2]} = \E{[\Bar{g}^2]} + \Phi$, and by definition $\E{[g^2]} \leq \Tilde{v}$, so that $\E{[\Bar{g}^2]} \leq \Tilde{v} - \Phi$ and $\E{[\Bar{g}^2]} \leq C^2$. With these conditions, by the same steps as in Lemma 5.1 \citep{défossez2022simple} we have,
    \begin{align*}
        \E{\bigg[\frac{\nabla_i F(\theta_{t-1}) \nabla_i f_t^{p}(\theta_{t-1})}{\sqrt{v_{t, i}^{p} -\Phi }}\bigg]} &\geq \frac{(\nabla_i F(\theta_{t-1}))^2}{2\sqrt{\Tilde{v}_{t,i}^{p}-\Phi}} \\ &- 2C\E{\bigg[ \frac{(\nabla_i f_t^{p}(\theta_{t-1}))^2}{v_{t,i}^{p}-\Phi} \bigg]}.
    \end{align*}
    Substituting the result from Lemma \ref{lemma:gaussian_concentration_on_z} we get with probability of at least $(1-\alpha)$, $0< \alpha < 1 $,
    \begin{align*}
        &\E{\bigg[\frac{\nabla_i F(\theta_{t-1}) \nabla_i f_t^{p}(\theta_{t-1})}{\sqrt{v_{t, i}^{p} + \epsilon }}\bigg]} \\
        &\geq \frac{(\nabla_i F(\theta_{t-1}))^2}{2\sqrt{\sum_{j=0}^{t-1} \beta_2^{j}\delta^2 - \Phi}} - 2C\E{\bigg[ \frac{(\nabla_i f_t^{p}(\theta_{t-1}))^2}{v_{t,i}^{p}+ \epsilon} \bigg]},
    \end{align*}
    where $\delta$ is in the form as in Lemma \ref{lemma:gaussian_concentration_on_z}.
    Summing over all steps $t \in [T]$ and taking the complete expectation we have,
    \begin{align*}
        \E{[F(\theta_T)]} &\leq F(\theta_0) - \frac{\eta \sum_{t=0}^{T-1} \E{[\norm{\nabla F(\theta_{t-1})}^2]}}{2\sqrt{\delta^2 - \Phi}} \\ &+ \bigg(\frac{2\eta C}{\sqrt{1-\beta_2}} + \frac{\eta^2 L}{2(1-\beta_2)}\bigg) \sum_{t=0}^{T-1} \E{[\lVert u_{t} \rVert^2]},
    \end{align*}
    where $\delta$ is in the form as in Corollary \ref{lemma:union_bound_on_v}.
    We bound $\sum_{t=0}^{T-1} \E{[\lVert u_t \rVert^2]}$ using the similar approach as in Lemma 5.2 \citep{défossez2022simple}. Let $a_t=(\nabla f_t^{p})^2$ and $b_t = \sum_{j=1}^{t}\beta_2^{t-j}a_j$, then $\sum_{t=1}^{T} \E{[\lVert u_t \rVert^2]} = \sum_{i=1}^{d} \E{[\sum_{t=1}^{T} \frac{a_{t, i}}{b_{t,i} - \Phi}]}$. Given $\ln$ is increasing and $b_t > a_t > 0$,
    \begin{align*}
        \frac{a_{t, i}}{b_{t,i} - \Phi} &\leq \ln{\big(1/(1-\frac{a_{t,i}}{b_{t,i}-\Phi})\big)} \\
        &= \ln{\big( \frac{b_{t,i}-\Phi}{b_{t-1,i}-\Phi} \big)} + \ln{\big( \frac{b_{t-1,i}-\Phi}{\beta_2 b_{t-1,i}-\Phi} \big)}
    \end{align*}
    \begin{align*}
        \sum_{t=1}^{T} \frac{a_{t, i}}{b_{t,i} - \Phi} &=  \sum_{t=1}^{T} \ln{\big( \frac{b_{t,i}-\Phi}{b_{t-1,i}-\Phi} \big)} +  \sum_{t=1}^{T} \ln{\big( \frac{b_{t-1,i}-\Phi}{\beta_2 b_{t-1,i}-\Phi} \big)} \\
        &\leq \ln{\big( 1-\frac{b_{T,i}}{\Phi} \big)} - T\ln{(\beta_2)},
    \end{align*}
    where the last inequality is because $b_{0}=0$ and $b_t > \Phi \; \forall t$.
    Substituting the result in we get,
    \begin{align*}
        \sum_{t=0}^{T-1} \E{[\lVert u_{t, i} \rVert^2]} &= \sum_{i=1}^{d} \E{\bigg[ \sum_{t=0}^{T-1} u_{t, i}^2 \bigg]} \\ &\leq \sum_{i=1}^{d} \E{[\ln{(1-v_{T}^p/\Phi)} - T\ln{(\beta_2)}]} \\
        &\leq \sum_{i=1}^{d} \big(\ln{(1-\frac{(C^2+\Phi)(1-\beta_2^{T})}{(1-\beta_2)\Phi}) - T\ln{(\beta_2)}}\big)\\
        &\leq d\big(\ln{(1-\frac{C^2+\Phi}{(1-\beta_2)\Phi}) - T\ln{(\beta_2)}}\big).
    \end{align*}
    Substituting in the result and rearrange the terms gives the final result.
\end{proof}

\paragraph{Discussion on the convergence bound.}
Comparing the convergence rate between the two algorithm is not straight-forward giving that one has a slight advantage in $\E{[(\nabla F) u_t]}$, the expectation of the update direction deviating from the true descent direction, and a disadvantage in $\E{\norm{u_t}^2}$, the expectation of the update size. We discuss an approximate comparison between $\E{\norm{u_t}^2}$. Let $u_t^{\textnormal{DP-Adam}}=\frac{g}{\sqrt{v+\epsilon}}$, $u_t^{\textnormal{DP-AdamBC}}=\frac{g}{\sqrt{v-\Phi}}$, with the data and noise distribution we can consider the numerator and denominator as two random variables such that for each dimension $i$, for DP-Adam we have $\E{[u_{t, i}^2]} = \E{[A/B]}$, for DP-AdamBC we have $\E{[u_{t, i}^2]} = \E{[A/B']}$, where $A$ is the random variable for $t$-step gradient $g^2$, $B$ is the random variable for $v+\epsilon$, $B'$ is the random variable for $v-\Phi$. It would be difficult to derive closed form results for expectation on ratios of random variables without specific assumption, but taking a second-order Taylor expansion approximately gives $\E{[\frac{X}{Y}]} \approx \frac{\E{[X]}}{\E{[Y]}}-\frac{Cov(X,Y)}{\E{[Y]}^2}+\frac{\Var{(Y)}}{\E{[Y]}^3}$. Since $\Var{(B)}=\Var{(B')}$, the differences between the approximated $\E{[A/B]}$ and $\E{[A/B']}$ is only in the denominators which only differs by constant $\Phi$. Therefore, giving that $\E{[(\nabla F) u_t]}$ only differs by constant terms, and $\E{\norm{u_t}^2}$ are approximately similar, we believe qualitatively there is no large difference in the convergence rate between the two algorithms under such analysis settings.

\begin{proposition}[Convergence of DP-AdamBC and DP-Adam With momentum]
    Given the above assumptions and the defined update rule with $0 < \beta_2 < 1$, $0 \leq \beta_1 < \beta_2 $, $\eta_t = \eta (1-\beta_1) \sqrt{\frac{1-\beta_2^t}{1-\beta_2}}, \eta>0$, $\forall \alpha$ s.t. $0<\alpha<1$, we have $\forall T \in \sN^*$ such that $T>\frac{\beta_1}{1-\beta_1}$ and with $\Tilde{T}=T-\frac{\beta_1}{1-\beta_1}$, let $\nu^{*} = (2\beta_2^2\Phi\sqrt{(\beta_2^t-1)/ (\beta_2^2-1)})$, $b^{*}=4\beta_2\Phi$, \\
    \\
    (DP-Adam)
    \begin{align*}
        &\E{[\norm{\nabla F(\theta_t)}^2]} \leq \frac{2\delta(F_0-F_*)}{\eta \Tilde{T}} \\ &+ E\bigg( \ln{\bigg( 1+\frac{\delta^2}{\epsilon(1-\beta_2)} \bigg)} -T\log{(\beta_2)}\bigg),
    \end{align*}
    \begin{align*}
        E &= \frac{\eta d L(1-\beta_1)\delta}{(1-\beta_1/\beta_2)(1-\beta_2)} + \frac{2\eta^2 d L^2 \beta_1}{(1-\beta_1/\beta_2)(1-\beta_2)^{3/2}} \\&+ \frac{12d\delta^2\sqrt{1-\beta_1}}{(1-\beta_1/\beta_2)^{3/2}\sqrt{1-\beta_2}},
    \end{align*}
    \\
    (DP-AdamBC)
    \begin{align*}
        &\E{[\norm{\nabla F(\theta_t)}^2]} \leq \frac{2\sqrt{\delta^2 - \Phi}(F_0-F_*)}{\eta\Tilde{T}}\\ &+ E\bigg( \ln{\bigg( 1-\frac{\delta^2}{\Phi(1-\beta_2)} \bigg)} -T\log{(\beta_2)}\bigg),
    \end{align*}
    \begin{align*}
        E &= \frac{\eta d L(1-\beta_1)\sqrt{\delta^2 - \Phi}}{(1-\beta_1/\beta_2)(1-\beta_2)} + \frac{2\eta^2 d L^2 \beta_1}{(1-\beta_1/\beta_2)(1-\beta_2)^{3/2}} \\&+ \frac{12d(\delta^2 - \Phi)\sqrt{1-\beta_1}}{(1-\beta_1/\beta_2)^{3/2}\sqrt{1-\beta_2}},
    \end{align*}
    \begin{align*}
        \delta \geq
        \begin{cases}
            \mu^{*} + \sqrt{\ln{(1/\frac{\alpha}{2T})}(2\nu^{*2})}
            & 0 \leq \delta  \leq \frac{\nu^{*2}}{b^{*}} \\
            \mu^{*} + \ln{(1/\frac{\alpha}{2T})}2b^{*}
            & \delta   \geq \frac{\nu^{*2}}{b^{*}}.
        \end{cases}
    \end{align*}
\end{proposition}

\paragraph{DP-Adam.}
\begin{proof}
    Let $G_t = \nabla F(\theta_{t-1})$, $g_t = \nabla f_t^{p}(\theta_t-1)$, by Assumption \ref{assmption:f_smooth} we have,
    \[
        F(\theta_t) \leq F(\theta_{t-1}) - \eta_t G_t^T u_t + \frac{\eta_t^2 L}{2} \norm{u_t}^2, \; u_t = \frac{m_t^{p}}{\sqrt{v_t^{p}+\epsilon}},
    \]
    where $m_t^{p}$ and $v_t^{p}$ are the first and second moment estimated from privatized gradient which makes $u_t$ the update of DP-Adam with momentum. Taking the expectation over past steps we get,
    \[
        \E{[F(\theta_t)]} \leq \E{[F(\theta_{t-1})]} - \eta_t \E{[G_t^T u_t]} + \frac{\eta_t^2 L}{2} \E{[\norm{u_t}^2]}.
    \]
    We bound $\E{[G_t^T u_t]}$ using a similar approach as in Lemma A.1 \citep{défossez2022simple} with the following key steps. For index $0 \leq k \leq t-1$, we first decompose $G_t^T u_t$ as,
    \begin{align*}
        \sum_{i\in[d]}G_{t,i}\frac{m_{t,i}^{p}}{\sqrt{v_{t,i}^{p}+\epsilon}} &= \underbrace{\sum_{i\in[d]} \sum_{k=0}^{t-1} \beta_1^{k}G_{t-k, i}\frac{g_{t-k,i}}{\sqrt{v_{t,i}^{p}+\epsilon}}}_{A} \\ &+ \underbrace{\sum_{i\in[d]} \sum_{k=0}^{t-1} \beta_1^{k}(G_{t, i}-G_{t-k, i})\frac{g_{t-k,i}}{\sqrt{v_{t,i}^{p}+\epsilon}}}_{B}.
    \end{align*}
    We first bound $B$ with the Gaussian concentration bound on $v_{t,i}^{p}$ using the similar approach as in Equation (A.13) in \citet{défossez2022simple}. Let $\lambda = \frac{\sqrt{1-\beta_1}}{2\delta}$ where $\delta$ is in the form as in Lemma \ref{lemma:gaussian_concentration_on_z}, $x = |G_{t,i}-G_{t-k, i}|$, $y=\frac{g_{t-k,i}}{\sqrt{v_{t,i}^{p}+\epsilon}}$, following the same steps we get,
    \begin{align*}
        |B| &\leq \frac{\eta_t^2 L^2\sqrt{1-\beta_1}}{4\delta}\Big( \sum_{l=1}^{t-1} \norm{u_{t-l}}^2 \sum_{k=l}^{t-1} \beta_1^k \sqrt{k} \Big)\\ &+ \frac{\delta}{\sqrt{1-\beta_1}}\Big( \sum_{k=0}^{t-1} (\frac{\beta_1}{\beta_2})^k \norm{U_{t-k}}^2 \Big).
    \end{align*}
    Let $\Tilde{v}_{t, k}^{p} = \beta_2^{k} v_{t-k} + \E{[\sum_{j=t-k+1}^{t} \beta_2^{t-j} g_{j}^2]}$ be the second moment estimate with last $k$ gradients replaced by their expected value, since by definition $\Tilde{v}_{t, k+1}^{p} + \epsilon \geq \E{[\sum_{j=t-k}^{t} \beta_2^{t-j}g_{j}^2]}$ and $v_{t}^{p} + \epsilon \geq \sum_{j=t-k}^{t} \beta_2^{t-j}g_{j}^2$ and with the result of Lemma \ref{lemma:gaussian_concentration_on_z}, following the same steps as in \citep{défossez2022simple},
    \begin{align*}
        &\E{[A]} \geq \frac{1}{2} \bigg( \sum_{i\in[d]} \sum_{k=0}^{t-1} \beta_1^k \E{\bigg[ \frac{G_{t-k,i}^2}{\sqrt{\Tilde{v}_{t, k+1,i}+\epsilon}} \bigg]}  \bigg)\\ &- 2\sqrt{\frac{\delta}{1-\beta_1}} \bigg( \sum_{i\in[d]} \sum_{k=0}^{t-1} \big(\frac{\beta_1}{\beta_2}\big)^k\E{[\norm{U_{t-k}}^2]} \bigg).
    \end{align*}
    Then combing the results for $A$ and $B$ gives the bound on $\E{[G_t^T u_t]}$. Let $\Omega_t=\sqrt{\sum_{j=0}^{t-1}\beta_2^{j}}$, by Assumption \ref{assmption:f_bounded_below}, summing over all steps $t$ and reorganizing the terms we get,
    \begin{align*}
        &\frac{\sum_{t=1}^{T}\frac{\eta_t}{\Omega_t} \sum_{k=0}^{t-1}\beta_1^k \E{[\norm{G_{t-k}}^2]}}{2\delta} \leq F(\theta_0) - F_{*} \\ &+ \frac{\eta_N^2L}{2}\sum_{t=1}^{T} \E{[\norm{u_t}^2]} + \frac{\eta_{T}^3L^2 \sqrt{1-\beta_1}}{4\delta}\sum_{t=1}^{T} \sum_{l=1}^{t-1} \beta_1^k\sqrt{k} \\& + \frac{3\eta_{T}\delta}{\sqrt{1-\beta_1}}\sum_{t=1}^{T}\sum_{k=0}^{t-1}\big(\frac{\beta_1}{\beta_2}\big)^k \sqrt{k+1} \E{[\norm{U_{t-k}}^2]},
    \end{align*}
    where $\delta$ is in the form as in Corollary \ref{lemma:union_bound_on_v}.
    Bounding $\E{[\norm{u_t}^2]}$ is similar to the steps in Lemma A.2 \citep{défossez2022simple} which we have,
    \begin{gather*}
        \sum_{t=1}^{T} \E{[\norm{u_t}^2]} \leq \frac{\sum_{i \in [d]} \ln{\big( 1+\frac{v_{T,i}}{\epsilon} \big) - T \log{(\beta_2)}}}{(1-\beta_1)(1-\beta_1/\beta_2)} ,\\
        v_{T,i} \leq \frac{(C + \sqrt{-\ln{\frac{\alpha}{2T}}(2\Phi)})^2}{1-\beta_2}.
    \end{gather*}
    The rest of the proof rearranges the other terms with techniques including changing index and order of summation and is exactly the same as in \citep{défossez2022simple} which leads to the final result.
\end{proof}

\paragraph{DP-AdamBC.}
\begin{proof}
    We start with,
    \begin{gather*}
        \E{[F(\theta_t)]} \leq \E{[F(\theta_{t-1})]} - \eta_t \E{[G_t^T u_t]} + \frac{\eta_t^2 L}{2} \E{[\norm{u_t}^2]}, \\
        u_t = \frac{m_t^{p}}{\sqrt{v_t^{p}-\Phi}},
    \end{gather*}
    where $u_t$ is the update of DP-AdamBC with momentum. To bound $\E{[G_t^T u_t]}$ we first decompose the quantity as follows,
    \begin{align*}
        \E{[G_t^T u_t]} &= \sum_{i\in[d]} \E{\bigg[ \frac{G_{t,i}m_{t,i}^p}{\sqrt{v_{t,i}^p - \Phi}} \bigg]} \\&= \sum_{i\in [d] } \E{\bigg[ \frac{G_{t,i}m_{t,i}^c}{\sqrt{v_{t,i}^p - \Phi}} \bigg]} + \sum_{i\in[d]} \underbrace{\E{\bigg[ \frac{G_{t,i}\sum_{j=1}^{t}\beta_1^{t-i}z_{t,i}}{\sqrt{v_{t,i}^p - \Phi}} \bigg]}}_{0} \\
        &= \underbrace{\sum_{i\in [d] } \sum_{k=0}^{t-1} \beta_1^{k} G_{t-k,i} \E{\bigg[ \frac{\bar{g}_{t-k,i}}{\sqrt{v_{t,i}^p - \Phi}} \bigg]}}_{A} \\&+ \underbrace{\sum_{i\in [d] } \sum_{k=0}^{t-1} \beta_1^{k} (G_{t, i}-G_{t-k,i}) \E{\bigg[ \frac{\bar{g}_{t-k,i}}{\sqrt{v_{t,i}^p - \Phi}} \bigg]}}_{B}.
    \end{align*}
    Let $\Tilde{v}_{t, k}^{p} = \beta_2^{k} v_{t-k} + \E{[\sum_{j=t-k+1}^{t} \beta_2^{t-j} g_{j}^2]}$ be the second moment estimate with last $k$ gradients replaced by their expected value, and substituting result from Lemma \ref{lemma:gaussian_concentration_on_z},
    \begin{align*}
        &\E{[G_t^T u_t]} \geq \frac{1}{2}\bigg( \sum_{i \in [d]} \sum_{k=0}^{t-1} \beta_1^k \E{\bigg[ \frac{G_{t-k,i}}{\sqrt{\Tilde{v}_{t,k+1,i}-\Phi}}\bigg]} \bigg) \\ &- \frac{\sqrt{1-\beta_1}\eta_t^2L^2}{4\sqrt{\delta^2 - \Phi}}\bigg( \sum_{l=1}\norm{u_{t-l}}^2\sum_{k=1}^{l-1}\beta_1^k\sqrt{k} \bigg) \\&- \frac{3\sqrt{\delta^2 - \Phi}}{\sqrt{1-\beta_1}} \bigg( \sum_{k=0}^{t-1} \big(\frac{\beta_1}{\beta_2} \big)^2 \norm{U_{t-k}}^2\bigg),
    \end{align*}
    where $\delta$ is in the form as in Lemma \ref{lemma:gaussian_concentration_on_z}.
    Let $\Omega_t=\sqrt{\sum_{j=0}^{t-1}\beta_2^{j}}$, by Assumption \ref{assmption:f_bounded_below}, summing over all steps $t$ and reorganizing the terms we get,
    \begin{align*}
        &\frac{\sum_{t=1}^{T}\frac{\eta_t}{\Omega_t} \sum_{k=0}^{t-1}\beta_1^k \E{[\norm{G_{t-k}}^2]}}{2\sqrt{\delta^2 - \Phi}} \leq F(\theta_0) - F_{*} \\&+ \frac{\eta_N^2L}{2}\sum_{t=1}^{T} \E{[\norm{u_t}^2]} + \frac{\eta_{T}^3L^2\sqrt{1-\beta_1}}{4\sqrt{\delta^2 - \Phi}}\sum_{t=1}^{T} \sum_{l=1}^{t-1} \beta_1^k\sqrt{k} \\
        &+ \frac{3\eta_{T}\sqrt{\delta^2 - \Phi}}{\sqrt{1-\beta_1}}\sum_{t=1}^{T}\sum_{k=0}^{t-1}\big(\frac{\beta_1}{\beta_2}\big)^k \E{[\norm{U_{t-k}}^2]},
    \end{align*}
    where $\delta$ is in the form as in Corollary \ref{lemma:union_bound_on_v}.
    Bounding $\E{[\norm{u_t}^2]}$ is similar to the steps in Lemma A.2 \citep{défossez2022simple} which we have,
    \begin{gather*}
        \sum_{t=1}^{T} \E{[\norm{u_t}^2]} \leq \frac{\sum_{i \in [d]} \ln{\big( 1-\frac{v_{T,i}}{\Phi} \big) - T \log{(\beta_2)}}}{(1-\beta_1)(1-\beta_1/\beta_2)} , \\
        v_{T,i} \leq \frac{(C + \sqrt{-\ln{\frac{\alpha}{2T}}(2\Phi)})^2}{1-\beta_2}.
    \end{gather*}
    The rest of the proof rearranges the other terms with techniques including changing index and order of summation and is exactly the same as in \citep{défossez2022simple} which leads to the final result.

\end{proof}

\section{Limitations}
We observe that DP-AdamBC improves performance of DP-Adam in the cases where both algorithms outperform DP-SGD, such as in text classification tasks with SNLI and QNLI. In cases where DP-SGD outperforms DP-Adam, such as in image classification with CIFAR10 (Figure \ref{fig:exp1}) and in node classification with obgn-arxiv (as reported in \citet{daigavane2022nodelevel}), DP-AdamBC tends to perform similarly to DP-Adam, with minor advantages.
Although the observed DP bias is quite concentrated around its mean $\Phi$, we note that $\gamma'$ in DP-AdamBC is an important hyperparameter that affects the choice of learning rate $\eta$ and affects the final performance. As such, our results are dependent on our efforts tuning parameters for each algorithm.
However, this also opens avenues for improvement. Since $\gamma'$ concentrates over steps $t$, we could apply a decreasing schedule for $\gamma'$ (and $\eta$, since a smaller learning rate is typically needed for smaller $\gamma'$) following the bound of Propositions \ref{prop:bound_fixed} and \ref{prop:bound_martingale} (and confirmed in the numerical analysis in \S \ref{apdix:bound}).

\end{document}